\documentclass[11pt]{article}
\usepackage{graphicx}
\usepackage{amsmath,amsthm,amscd}
\usepackage{amssymb}
\usepackage{fancybox}
\usepackage{latexsym}
\usepackage{fancyhdr}
\usepackage{lastpage}
\usepackage{slashbox}
\usepackage{url}
\usepackage{natbib}
\usepackage{lscape}
\usepackage{multirow}

\newtheorem{theorem}{Theorem}

\newtheorem{lemma}{Lemma}
\newtheorem{definition}{Definition}

\newtheorem{example}{Example}

  \usepackage{amsmath}	
\usepackage{bm}
\usepackage{color}

\newcommand{\indicator}[1]{[\![\,{#1}\,]\!]}
\def\ind{\indicator}

\def\sign{{\rm sign}}
\def\Rbb{\mathbb{R}}
\def\Real{\mathbb{R}}
\def\Nbb{\mathbb{N}}
\def\Ybin{\{+1,-1\}}
\def\Xcal{\mathcal{X}}
\def\Ecal{\mathcal{E}}
\def\Fcal{\mathcal{F}}
\def\Gcal{\mathcal{G}}
\def\Hcal{\mathcal{H}}
\def\Lcal{\mathcal{L}}
\def\Ncal{\mathcal{N}}
\def\Rcal{\mathcal{R}}
\def\Scal{\mathcal{S}}
\def\Ucal{\mathcal{U}}

\def\st{\mathrm{\text{subject to }}}
\def\conv{\mathrm{conv}}
\def\Ebb{\mathbb{E}}
\def\<{\langle}
\def\>{\rangle}
\newtheorem{assumption}{Assumption}

\def\x{{\bm x}}
\def\z{{\bm z}}

\def\w{{\bm w}}
\def\0{{\bm 0}}
\def\1{{\bm 1}}
\def\alphabold{{\bm \alpha}}
\def\betabold{{\bm \beta}}
\def\mubold{{\bm \mu}}
\def\xibold{{\bm \xi}}

\begin{document}
\title{A Conjugate Property between Loss Functions and Uncertainty Sets in Classification Problems}
\author{
  Takafumi Kanamori\\ Nagoya University \\ \tt{kanamori@is.nagoya-u.ac.jp}
  \and
   Akiko Takeda\\ Keio University \\ \tt{takeda@ae.keio.ac.jp}
  \and 
   Taiji Suzuki\\ The University of Tokyo\\ \tt{t-suzuki@mist.i.u-tokyo.ac.jp}
}
%
%
%
%
\date{}
\maketitle

\begin{abstract}
 In binary classification problems, mainly two approaches have been proposed; 
 one is loss function approach and the other is uncertainty set approach. 
 The loss function approach is applied to major learning algorithms such as support vector
 machine (SVM) and boosting methods. 
 The loss function represents the penalty of the decision
 function on the training samples. 
 In the learning algorithm, the empirical mean of the 
 loss function is minimized to obtain the classifier. 
 Against a backdrop of the development of mathematical programming, 
 nowadays learning algorithms based on loss functions are widely applied to real-world
 data analysis. In addition, statistical properties of such learning algorithms are
 well-understood based on a lots of theoretical works. 
 On the other hand, the learning method using the so-called uncertainty set is used in 
 hard-margin SVM, mini-max probability machine (MPM) and maximum margin MPM. 
 In the learning algorithm, firstly, the uncertainty set is defined for each binary label 
 based on the training samples. Then, 
 the best separating hyperplane between the two uncertainty sets  is employed as the
 decision function. This is regarded as an extension of the maximum-margin approach. 
 The uncertainty set approach has been studied as an application of robust
 optimization in the field of mathematical programming.  
 The statistical properties of learning algorithms with uncertainty sets have not been 
 intensively studied. 
 In this paper, we consider the relation between the above two approaches. 
 We point out that the uncertainty set is described by using the level set of
 the conjugate of the loss function. Based on such relation, we study statistical
 properties of learning algorithms using uncertainty sets. 
\end{abstract}

\section{Introduction}
\label{sec:Introduction}
In classification problems, the goal is to predict output labels for given input vectors. 
For this purpose, a decision function defined on the input space is estimated from
training samples. The output value of the decision function is used for the label 
prediction. 
In binary classification problems, the label is predicted by the sign of the decision
function. 

Many learning algorithms use loss functions to measure the penalty of misclassifications. 
The decision function minimizing the empirical mean of the loss function over training
samples is employed as the estimator  
\citep{cortes95:_suppor_vector_networ,nc:Schoelkopf+Smola+Williamson:2000,FreundSchapire97,Hastie_etal01}. 
For example, hinge loss, exponential loss and logistic loss are used for 
support vector machine (SVM), Adaboost and logistic regression, respectively. 
Especially in the binary classification tasks, statistical properties of learning
algorithms based on loss functions are well-understood due to intensive recent works. 
See \cite{bartlett06:_convex_class_risk_bound,steinwart05:_consis_of_suppor_vector_machin,
steinwart03:_optim_param_choic_suppor_vector_machin,Schapire_etal98,zhang04:_statis,Vapnik98}
for details. 

As another approach, the maximum-margin criterion is also applied for the statistical
learning. Under the maximum-margin criterion, the best separating hyperplane between the
two output labels is employed as the decision function. 
In hard-margin SVM \citep{Vapnik98}, a convex-hull of input vectors for each binary label
is defined, and the maximum-margin between the two convex-hulls is considered. 
For the non-separable case, $\nu$-SVM provides a similar picture
\citep{nc:Schoelkopf+Smola+Williamson:2000,ICML:Bennett+Bredensteiner:2000}. 
In $\nu$-SVM, the so-called reduced convex-hull which is a subset of the original
convex-hull is used for the learning. A reduced convex-hull is defined for each label, 
and the best separating hyperplane between the two reduced convex-hulls is employed as the 
decision function. Not only polyhedral sets such as the convex-hull of finite input points
but also ellipsoidal sets are applied for classification problems
\citep{Lanckriet:2003:RMA:944919.944934,Nath07}.  
In this paper, the set used in 
the maximum-margin criterion is 
referred to as \emph{uncertainty set}. This term is borrowed from robust optimization
in mathematical programming \citep{book:Ben-Tal+etal:2009}. 

There are some works in which the statistical properties of the learning based on the 
uncertainty set are studied. 
For example, \cite{Lanckriet:2003:RMA:944919.944934} proposed minimax probability machine
(MPM) using the ellipsoidal uncertainty sets, and studied statistical properties under the
worst-case setting. 
In the statistical learning using uncertainty set, the main concern is to develop
optimization algorithms under the maximum margin criterion
\citep{mavroforakis06:_suppor_vector_machin_svm}. 
So far, statistical properties of the learning algorithm using uncertainty sets 
have not been intensively studied compared to the learning using loss functions. 

The main purpose of this paper is to study the learning algorithm using the uncertainty
set. We focus on the relation between the loss function and the uncertainty set. 
We show that the uncertainty set is described by using the conjugate function of the loss
function. 
For given uncertainty set, we construct the corresponding loss function. 
We study the statistical properties of the learning algorithm using the uncertainty set by
applying theoretical results on the loss function approach. 
Then, we establish the statistical consistency of learning algorithms using the uncertainty set. 
We point out that in general the maximum margin criterion for a fixed uncertainty set 
does not provide accurate decision functions. 
We need to introduce a parametrized uncertainty set by the one-dimensional parameter which
specifies the size of the uncertainty set. 
We show that a modified maximum margin criterion with the parametrized uncertainty set
recovers the statistical consistency. 

The paper is organized as follows. 
In Section \ref{sec:Preliminaries}, we introduce the existing method based on the
uncertainty set. 
In Section \ref{sec:Loss_and_Uncertainty}, 
we investigate the relation between loss functions and uncertainty sets. 
Section \ref{sec:Revision_Uncertainty_Set} is devoted to illustrate a way of
revising the uncertainty set to recover nice statistical properties. 
In Section \ref{sec:Learning_Algorithm}, we present a kernel-based learning algorithm
with uncertainty sets. 
In Section \ref{sec:Statistical_Properties}, we prove that the proposed algorithm has the
statistical consistency. 
Numerical experiments are shown in Section \ref{sec:Numerical_Studies}. 
We conclude in section \ref{sec:Conclusion}. 
Some proofs are shown in Appendix. 

We summarize some notations to be used throughout the paper. 
The indicator function is denoted as $\ind{A}$, i.e., $\ind{A}$ equals $1$ if $A$ is true, 
and $0$ otherwise. 
The column vector $\x$ in the Euclidean space is described in bold face. 
The transposition of $\x$ is denoted as $\x^{T}$.  
The Euclidean norm of the vector  $\x$ is expressed as $\|\x\|$. 
For a set $S$ in a linear space, the convex-hull of $S$ is denoted as $\conv{S}$ or $\conv(S)$. 
The number of elements in the set $S$ is denoted as $|S|$. 
The expectation of the random variable $Z$ w.r.t. the probability distribution $P$
is described as $\Ebb_P[Z]$. We will drop the subscript $P$ as $\Ebb[Z]$, when it is clear 
from the context. 
The set of all measurable functions on the set $\Xcal$ is denoted by $L_0(\Xcal)$ or $L_0$
for short. 
The supremum norm of $f\in{L_0}$ is denoted as $\|f\|_\infty$. 
For the reproducing kernel Hilbert space
$\Hcal$, $\|f\|_\mathcal{H}$ is the norm of $f\in\Hcal$ defined from the inner product
$\<\cdot,\cdot\>_\Hcal$ on $\Hcal$.

\section{Preliminaries}
\label{sec:Preliminaries}
We define $\Xcal$ as the input space and $\Ybin$ as the set of binary labels. 
Suppose that the training samples $(x_1,y_1),\ldots,(x_m,y_m)\in\Xcal\times\Ybin$ 
are drawn i.i.d. according to a probability distribution $P$ on $\Xcal\times\Ybin$. 
The goal is to estimate a decision function $f:\Xcal\rightarrow\Rbb$ from a set of
functions $\Fcal$, such that the sign of $f(x)$ provides an accurate prediction of the
unknown binary label associated with the input $x$ under the probability distribution $P$. 
In other word, for the estimated decision function $f$,
the probability of $\sign(f(x))\neq{y}$ is expected to be as small as possible. 
In this article, the composite function of the sign function and the decision function, 
$\sign(f(x))$, is referred to as classifier. 

\subsection{Learning with loss functions}
\label{subsec:Loss_function_approach}
In binary classification problems, 
the prediction accuracy of the decision function $f$ is
measured by the 0-1 loss $\ind{yf(x)\leq0}$ which equals 
$1$ when the sign of $f(x)$ is different from $y$ and $0$ otherwise. 
The average prediction performance of the decision function $f$ is evaluated by the
expected 0-1 loss, i.e., 
\begin{align}
 \Ecal(f)=\Ebb[\,\ind{y f(x)\leq0}\,].
\end{align}
The Bayes risk $\Ecal^*$ is defined as the minimum value of the expected 0-1 loss 
over all the measurable functions on $\Xcal$, 
\begin{align}
 \Ecal^*=\inf\{ \Ecal(f)\,:\,f\in{}L_0 \}. 
\label{eqn:Bayes-risk}
\end{align}
Bayes risk is the lowest achievable error rate under the probability $P$. 
Given the set of training samples, $T=\{(x_1,y_1),\ldots,(x_m,y_m)\}$, 
the empirical 0-1 loss is denoted by
\begin{align}
 \widehat{\Ecal}_T(f)=\frac{1}{m}\sum_{i=1}^{m}\ind{y_if(x_i)\leq 0}. 
 \label{eqn:empirical-error-rate}
\end{align}
The subscript $T$ in $\widehat{\Ecal}_T(f)$ is dropped if it is clear from the context. 

In general, minimization of $\widehat{\Ecal}_T(f)$ is considered as a hard problem
\citep{arora97:_hardn_approx_optim_lattic_codes}. 
The main difficulty is considered to come from non-convexity of the 0-1 loss
$\ind{yf(x)\leq0}$ as the function of $f$. 
Hence, many learning algorithms use a surrogate loss of the 0-1 loss in order to make the
computation tractable.  
For example, SVM uses the hinge loss, $\max\{1-yf(x),0\}$, and Adaboost uses the
exponential loss, $\exp\{-yf(x)\}$. Both the hinge loss and the exponential loss are convex
in $f$, and they provide an upper bound of the 0-1 loss. 
Thus, the minimizer under the surrogate loss is also expected to minimize the 0-1 
loss. The quantitative relation between the 0-1 loss and the surrogate loss was studied by
\cite{bartlett06:_convex_class_risk_bound}. 

To avoid overfitting of the estimated decision function to training samples, 
the regularization is considered. 
By adding the regularization term such as the squared norm of the decision function 
to the empirical surrogate loss, the complexity of the estimated classifier is restricted.  
The balance between the regularization term and the surrogate loss is adjusted by the
regularization parameter~
\citep{evgeniou99:_regul_networ_suppor_vector_machin,steinwart05:_consis_of_suppor_vector_machin}.    
Then, the deviation of the empirical 0-1 loss and the expected 0-1 loss is controlled by
the regularization. When both the regularization term and the surrogate loss are convex, 
the computational tractability of the statistical learning is retained. 


\subsection{Learning with uncertainty sets}
Besides statistical learning using loss functions, 
there is another approach to the classification problems, i.e., 
statistical learning based on the so-called \emph{uncertainty set}. 
We briefly introduce the basic idea of the uncertainty set. 
We assume that $\Xcal$ is a subset of Euclidean space. 

In robust optimization problems \citep{book:Ben-Tal+etal:2009}, the uncertainty set
describes uncertainties or ambiguities included in optimization problems. 
The parameter in the optimization problem may not be precisely determined. 
Instead of the precise information, we have an uncertainty set which probably includes 
the parameter in the optimization problem. 
The worst-case setting is employed to solve the robust optimization problem with the
uncertainty set. 

The statistical learning with uncertainty set is considered as an application of the
robust optimization to classification problems. 
In classification problems, the uncertainty set is designed such that most training
samples are included in the uncertainty set with high probability. 
We prepare an uncertainty set for each binary label. 
For example, $\Ucal_p$ and $\Ucal_n$ are the confidence regions such that the conditional
probabilities, $P(\x\in\Ucal_p|y=+1)$ and $P(\x\in\Ucal_n|y=-1)$, are equal to $0.95$. 
As the other example, the uncertainty set $\Ucal_p$ (resp. $\Ucal_n$) consists of the
convex-hull of input vectors in training samples having the positive (resp. negative)
label. 
The convex-hull of data points is used in hard margin SVM~
\citep{ICML:Bennett+Bredensteiner:2000}. The ellipsoidal uncertainty set 
is also used for the robust classification under the worst-case
setting~\citep{Lanckriet:2003:RMA:944919.944934,Nath07}. 

Based on the uncertainty set, 
we estimate the linear decision function $f(\x)=\w^T\x+b$. 
Here, we consider the \emph{minimum distance problem} 
\begin{align}
 \min_{\x_p,\x_n}\|\x_p-\x_n\|\quad\st\ \x_p\in\Ucal_p,\,\x_n\in\Ucal_n. 
 \label{eqn:min-distance-prob}
\end{align}
Let $\x_p^*$ and $\x_n^*$ be optimal solutions of \eqref{eqn:min-distance-prob}. 
Then, the normal vector of the decision
function, $\w$, is estimated by $c(\x_p^*-\x_n^*)$, where $c$ is a positive real number. 
Figure \ref{fig:uncertaintyset_approach} illustrates the estimated decision boundary. 
When both $\Ucal_p$ and $\Ucal_n$ are compact subsets satisfying
$\Ucal_p\cap\Ucal_n=\emptyset$, the estimated normal vector cannot be the null vector. 
The minimum distance problem appears in the hard margin SVM
\citep{Vapnik98,ICML:Bennett+Bredensteiner:2000}, 
$\nu$-SVM \citep{nc:Schoelkopf+Smola+Williamson:2000,CriBur00}
and the learning algorithms proposed by 
\cite{Nath07,mavroforakis06:_suppor_vector_machin_svm}. 
In Section \ref{subsec:Uncertainty_Set_nuSVM}, we briefly introduce the relation between
$\nu$-SVM and the minimum distance problem. 
In minimax probability machine (MPM) proposed by \cite{Lanckriet:2003:RMA:944919.944934}, 
the other criterion is applied to estimate the linear decision function, though the
ellipsoidal uncertainty set plays an important role also in their algorithm. 
 \begin{figure}
  \centering
  \includegraphics[scale=0.6]{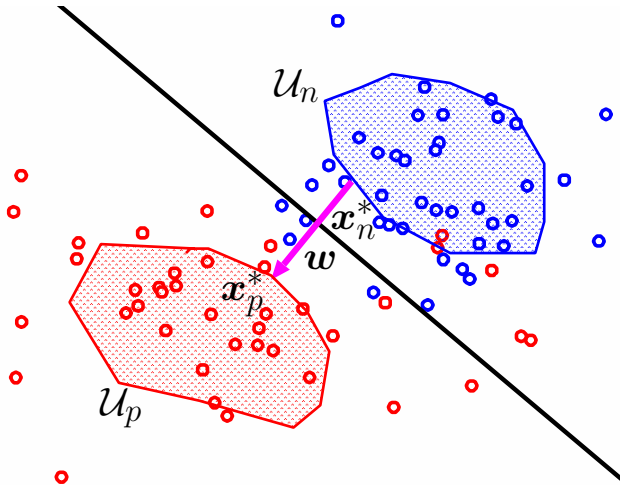} 
  \caption{The estimated decision boundary based on the minimum distance problem with 
  the uncertainty sets $\Ucal_p$ and $\Ucal_n$. }
  \label{fig:uncertaintyset_approach}
 \end{figure}

The minimum distance problem is equivalent with the maximum margin principle
\citep{Vapnik98,ICML:Bennett+Bredensteiner:2000}. 
When the bias term $b$ in the linear decision function is estimated such
that the decision boundary bisects the line segment connecting $\x_p^*$ and $\x_n^*$, the estimated
decision boundary achieves the maximum margin between the uncertainty sets, 
$\Ucal_p,\,\Ucal_n$. 
According to \cite{takeda12:_unified_robus_class_model}, 
we explain how the maximum margin is connected with the minimum distance. 
Suppose that $\Ucal_p$ and $\Ucal_n$ are convex subsets and that
$\Ucal_p\cap\Ucal_n=\emptyset$ holds. Then, the margin of two uncertainty sets along the
direction of $\w$ is given as 
\begin{align*}
 \min\bigg\{ \frac{\w^T\x_p-\w^T\x_n}{\|\w\|}:\x_p\in\Ucal_p,\,\x_n\in\Ucal_n\bigg\}. 
\end{align*}
The maximum margin criterion is described as 
\begin{align*}
 \max_{\w\neq\0}\min\bigg\{
 \frac{\w^T\x_p-\w^T\x_n}{\|\w\|}:\x_p\in\Ucal_p,\,\x_n\in\Ucal_n\bigg\}
 =\min\{\|\x_p-\x_n\|:\x_p\in\Ucal_p,\,\x_n\in\Ucal_n\}. 
\end{align*}
The equality above follows from the minimum norm duality 
\citep{luenberger97:_optim}.

\section{Relation between Loss Functions and Uncertainty Sets}
\label{sec:Loss_and_Uncertainty}
We study the relation between loss functions and uncertainty sets. First, we introduce the
relation in $\nu$-SVM according to \cite{CriBur00} and
\cite{ICML:Bennett+Bredensteiner:2000}. 
Then, we present an extension of $\nu$-SVM to
investigate a generalized relation between loss functions and uncertainty sets. 

\subsection{Uncertainty Set in $\nu$-SVM}
\label{subsec:Uncertainty_Set_nuSVM}
Suppose that the input space $\Xcal$ is a subset of Euclidean space $\Rbb^d$. 
We consider the linear decision function, $f(\x)=\w^T\x+b$, where the normal vector
$\w\in\Rbb^d$ and the bias term $b\in\Rbb$ are to be estimated based on observed training
samples. By applying the kernel trick
\citep{berlinet04:_reprod_hilber,book:Schoelkopf+Smola:2002}, we obtain rich statistical
models for the decision function, while keeping the computational tractability.  

In $\nu$-SVM, the classifier is estimated as the optimal solution of
\begin{align}
 \min_{w,b,\rho}\
 \frac{1}{2}\|\w\|^2-\nu\rho+\frac{1}{m}\sum_{i=1}^{m}\max\{ \rho-y_i(\w^T\x_i+b),\,0 \}, 
 \quad \w\in\Rbb^d,\ b\in\Rbb,\ \rho\in\Rbb, 
 \label{eqn:nu-svm}
\end{align}
where $\nu\in(0,1)$ is a prespecified constant which has the role of the regularization
parameter. 
As \cite{nc:Schoelkopf+Smola+Williamson:2000} pointed out, the parameter $\nu$ controls
the margin errors and number of support vectors. 
In $\nu$-SVM, a variant of the hinge loss, 
$\max\{\rho-y_i(\w^T\x_i+b),\,0\}$, 
is used as the surrogate loss. In the original formulation of $\nu$-SVM, 
the non-negativity constraint, $\rho\geq0$, is introduced. 
As shown by \cite{CriBur00}, we can confirm that the non-negativity constraint is
redundant. 
Indeed, for an optimal solution $\widehat{\w},\widehat{b},\widehat{\rho}$, we have
\begin{align*}
 -\nu\widehat{\rho}
 \leq 
 \frac{1}{2}\|\widehat{\w}\|^2-\nu\widehat{\rho}+\frac{1}{m}\sum_{i=1}^{m}
 \max\{\widehat{\rho}-y_i(\widehat{\w}^T\x_i+\widehat{b}),\,0\}
 \leq{}0, 
\end{align*}
where the last inequality comes from the fact that the parameter, 
$\w=\0,\,b=0,\,\rho=0$, is a feasible solution of \eqref{eqn:nu-svm}. 
As a result, we have $\widehat{\rho}\geq0$ for $\nu>0$. 

We briefly show that the dual problem of \eqref{eqn:nu-svm} yields the minimum distance
problem in which the reduced convex-hulls of training samples are used as uncertainty 
sets. See \cite{ICML:Bennett+Bredensteiner:2000} for details. 
The problem \eqref{eqn:nu-svm} is equivalent with 
\begin{align*}
 &\min_{\w,b,\rho,\xibold}\
 \frac{1}{2}\|\w\|^2-\nu\rho+\frac{1}{m}\sum_{i=1}^{m}\xi_i,\\
 &\st\  \xi_i\geq0,\ \xi_i\geq \rho-y_i(\w^T\x_i+b),\ i=1,\ldots,m. 
\end{align*}
Then, the Lagrangian function is defined as 
\begin{align*}
 L(\w,b,\rho,\xibold,\alphabold,\betabold)=
\frac{1}{2}\|\w\|^2-\nu\rho+\frac{1}{m}\sum_{i=1}^{m}\xi_i
 +\sum_{i=1}^{m}\alpha_i(\rho-y_i(\w^T\x_i+b)-\xi_i)-\sum_{i=1}^{m}\beta_i\xi_i, 
\end{align*}
where $\alpha_i,\,\beta_i,\,i=1,\ldots,m$ are non-negative Lagrange multipliers. 
For the observed training samples, 
we define $M_{p}$ and $M_{n}$ as the set of sample indices for each label, i.e., 
\begin{align}
 M_p=\{i~|~y_i=+1\},\quad M_n=\{i~|~y_i=-1\}.
 \label{eqn:label_index_set}
\end{align}
By applying min-max theorem, we have
\begin{align}
 &\phantom{=}
 \inf_{\w,b,\rho,\xibold}\sup_{\alphabold\geq\0,\betabold\geq\0} 
 L(\w,b,\rho,\xibold,\alphabold,\betabold)
 \nonumber\\
 &=
 \sup_{\alphabold\geq\0,\betabold\geq\0} \inf_{\w,b,\rho,\xibold}
 L(\w,b,\rho,\xibold,\alphabold,\betabold)
 \nonumber\\
 &=
 \sup
 \bigg\{
 -\frac{1}{2}\big\|\sum_{i=1}^{m}\alpha_iy_i\x_i\big\|^2
 ~:~\sum_{i=1}^{m}\alpha_i=\nu,\,\sum_{i=1}^{m}\alpha_iy_i=0,\,
 0\leq\alpha_i\leq\frac{1}{m}
 \bigg\}
 \nonumber\\ 
 &=
 -\frac{\nu^2}{8}\inf
 \bigg\{
 \big\|\sum_{i\in{M_p}}\gamma_i\x_i-\sum_{j\in{M_n}}\gamma_j\x_j \big\|^2
 \,:\,\sum_{i\in{M_p}}\gamma_i=\sum_{i\in{M_n}}\gamma_i=1,\,
 0\leq\gamma_i\leq\frac{2}{m\nu},i=1,\ldots,m
 \bigg\}, 
 \label{eqn:opt_reduced_convex-hull}
\end{align}
where the last equality is obtained by changing the variable from $\alpha_i$ to
$\gamma_i=2\alpha_i/\nu$. 
For the positive (resp. negative) label, we introduce the uncertainty set $\Ucal_p$
(reps. $\Ucal_n$) defined by the reduced convex-hull, i.e., 
\begin{align*}
 o\in\{p,n\},\quad
 \Ucal_o&=\bigg\{\sum_{i\in{M_o}}\gamma_i\x_i\,:\,
 \sum_{i\in{M_o}}\gamma_i=1,\ 0\leq\gamma_i\leq\frac{2}{m\nu},\,i\in{M_o}\bigg\}. 
\end{align*}
When the upper limit of $\gamma_i$ is less than one, the reduced convex-hull is a subset
of the convex-hull of training samples. 
We find that solving the problem \eqref{eqn:opt_reduced_convex-hull} is identical to
solving the minimum distance problem under the uncertainty set of 
the reduced convex-hulls, 
\begin{align*}
 \inf_{\x_p,\x_n}\|\x_p-\x_n\|\quad\st\ \x_p\in{\Ucal_p},\ \x_n\in{\Ucal_n}. 
\end{align*}
The representation based on the minimum distance problem provides an intuitive
understanding of the learning algorithm.

\subsection{Uncertainty Set Associated with Loss Function}
\label{sec:Uncertainty_Sets_Corresponding_Loss}
We consider general loss functions, and study the relation between the loss function and
the corresponding uncertainty set. 
Again, the decision function is defined as $f(\x)=\w^T\x+b$ on $\Rbb^d$. 
Let $\ell:\Rbb\rightarrow\Rbb$ be a convex and non-decreasing function. 
For the training samples, $(\x_1,y_1),\ldots,(\x_m,y_m)$, 
we propose a learning method in which the decision function is estimated by solving 
\begin{align}
 \inf_{\w,b,\rho}-2\rho+\frac{1}{m}\sum_{i=1}^{m}\ell(\rho-y_i(\w^T\x_i+b))
 \ \ \st\ \|\w\|^2\leq \lambda^2,\ b\in\Rbb,\,\rho\in\Rbb. 
 \label{eqn:general-loss}
\end{align}
The regularization effect is introduced by the constraint 
$\|\w\|^2\leq \lambda^2$, where $\lambda$ is the regularization parameter 
which may depend on the sample size. 

The statistical learning using \eqref{eqn:general-loss} is regarded as an extension of
$\nu$-SVM. To see this, we define $\ell(z)=\max\{2z/\nu,0\}$. 
Let $\widehat{\w},\widehat{b},\widehat{\rho}$ be an optimal solution of \eqref{eqn:nu-svm}
for a fixed $\nu\in(0,1)$. By comparing the optimality conditions of \eqref{eqn:nu-svm}
and \eqref{eqn:general-loss}, we can confirm that the problem \eqref{eqn:general-loss}
with $\lambda=\|\widehat{\w}\|$ has the same optimal solution as $\nu$-SVM. 

In the similar way as $\nu$-SVM, we derive the uncertainty set 
associated with the loss function $\ell$ in  \eqref{eqn:general-loss}. 
We introduce the slack variables $\xi_i,i=1,\ldots,m$ satisfying the inequalities
$\xi_i\geq\rho-y_i(\w^T\x_i+b),\,i=1,\ldots,m$. 
Then, the Lagrangian function of \eqref{eqn:general-loss} is given as 
\begin{align*}
 L(\w,b,\rho,\xibold,\alphabold,\mu)=
 -2\rho+\frac{1}{m}\sum_{i=1}^{m}\ell(\xi_i)
 +\sum_{i=1}^m\alpha_i(\rho-y_i(\w^T\x_i+b)-\xi_i)
 +\mu(\|\w\|^2-\lambda^2), 
\end{align*}
where $\alpha_1,\ldots,\alpha_m$ and $\mu$ are the non-negative Lagrange multipliers. 
The optimality conditions, 
\begin{align*}
 \frac{\partial{L}}{\partial\rho}=0,\ \text{and}\ \ 
 \frac{\partial{L}}{\partial{b}}=0
\end{align*}
and the non-negativity of $\alpha_i$ lead to the constraint on Lagrange multipliers, 
\begin{align*}
 \sum_{i\in{M_{p}}}\alpha_i=\sum_{i\in{M_{n}}}\alpha_i=1,\quad \alpha_i\geq0. 
\end{align*}
We define the conjugate function of $\ell(z)$ as
\begin{align*}
 \ell^*(x)=\sup_{z\in\Rbb}\{xz-\ell(z)\}. 
\end{align*}
Then, by applying min-max theorem, we have 
\begin{align}
 &\phantom{=}
 \inf_{\w,b,\rho,\xibold}
 \sup_{\alphabold\geq\0,\mu\geq0}L(\w,b,\rho,\xibold,\alphabold,\mu)
 \nonumber\\
 &=
 \sup_{\alphabold\geq\0,\mu\geq0}\inf_{\w,b,\rho,\xibold}
 L(\w,b,\rho,\xibold,\alphabold,\mu)
 \nonumber\\
 &=
 \sup_{\alphabold,\mu\geq0}\inf_{\w,\xibold}
 \bigg\{
 -\frac{1}{m}\sum_{i=1}^{m}(m\alpha_i\xi_i-\ell(\xi_i))
 -\sum_{i=1}^{m}\alpha_iy_i\x_i^T\w
 +\mu(\|\w\|^2-\lambda^2) 
 \nonumber\\ &\qquad\qquad\qquad 
 : \sum_{i\in{M_{p}}}\alpha_i=\sum_{i\in{M_{n}}}\alpha_i=1,\,\alpha_i\geq0
 \bigg\} \nonumber\\
 &=
 -\inf_{\alphabold,\mu\geq0}
 \bigg\{
 \frac{1}{m}\sum_{i=1}^{m}\ell^*(m\alpha_i)
 +\frac{1}{4\mu}\big\|\sum_{i=1}^{m}\alpha_iy_i\x_i\big\|^2+\mu\lambda^2
 : \sum_{i\in{M_{p}}}\alpha_i = \sum_{i\in{M_{n}}}\alpha_i=1,\,\alpha_i\geq0
 \bigg\} \nonumber\\
 &=
 -\inf_{\alphabold} 
 \bigg\{
 \frac{1}{m}\sum_{i=1}^{m}\ell^*(m\alpha_i)
 +\lambda
 \bigg\|\sum_{i\in{M_p}}\alpha_i\x_i-\sum_{i\in{M_n}}\alpha_i\x_i \bigg\|
 \,:\, \sum_{i\in{M_{p}}}\alpha_i=\sum_{i\in{M_{n}}}\alpha_i=1,\,\alpha_i\geq{0}
 \bigg\}. 
 \label{eqn:general-nu-svm-dual}
\end{align}
In Section \ref{sec:Statistical_Properties}, 
we present a rigorous proof that under some assumptions on $\ell(\xi)$, 
the min-max theorem works in the above Lagrangian function, i.e., there is no duality
gap. For each binary label, we define the parametrized uncertainty sets, 
$\Ucal_p[c]$ and $\Ucal_n[c]$, by 
\begin{align}
 o\in\{p,n\},\quad
 \Ucal_o[c]=\bigg\{\sum_{i\in{M_o}}\alpha_i\x_i
 \,:\, \alpha_i\geq0, \,\sum_{i\in{M_o}}\alpha_i=1,\ 
 \frac{1}{m}\sum_{i\in{M_o}}\ell^*(m\alpha_i)\leq{c} \bigg\}. 
 \label{eqn:uncertainty-set}
\end{align}
Then, the optimization problem in \eqref{eqn:general-nu-svm-dual} is represented by 
\begin{align}
\inf_{c_p,c_n,\z_p,\z_n}\!\! c_p+c_n+\lambda\|\z_p-\z_n\|\quad
 \st\ \z_p\in\Ucal_p[c_p],\,\z_n\in\Ucal_n[c_n],\ c_p,\,c_n\in\Rbb. 
 \label{eqn:RCM-representation}
\end{align}
Let $\widehat{\z}_p$ and $\widehat{\z}_n$ be the optimal solution of $\z_p$ and $\z_n$ 
in \eqref{eqn:RCM-representation}. 
Let $\widehat{\w}$ be an optimal solution of $\w$ in \eqref{eqn:general-loss}. 
The saddle point of the above min-max problem \eqref{eqn:general-nu-svm-dual} provides the 
relation between the $\widehat{\z}_p$, $\widehat{\z}_n$ and $\widehat{\w}$. 
Some calculation yields that, when $\widehat{\z_p}=\widehat{\z}_n$ holds,  
any vector such that $\|\widehat{\w}\|^2\leq\lambda^2$ satisfies the KKT condition of
\eqref{eqn:general-loss}. 
On the other hand, when $\widehat{\z}_p\neq\widehat{\z}_n$ holds, 
$\widehat{\w}$ is given by
$\widehat{\w}=\lambda(\widehat{\z}_p-\widehat{\z}_n)/\|\widehat{\z}_p-\widehat{\z}_n\|$. 
Hence, an optimal solution of the normal vector in the linear decision function 
is given as 
\begin{align}
 \widehat{\w}=
 \begin{cases}
  \displaystyle
  \frac{\lambda}{\|\widehat{\z}_p-\widehat{\z}_n\|}(\widehat{\z}_p-\widehat{\z}_n),
   & \widehat{\z}_p\neq\widehat{\z}_n,\\
  \displaystyle
  \0, & \widehat{\z}_p=\widehat{\z}_n. 
 \end{cases}
 \label{eqn:opt_w}
\end{align}
 We show a sufficient condition that the equality $\widehat{\z}_p=\widehat{\z}_n$ holds. 
 Suppose that $\Ucal_p[c_p]\cap\Ucal_n[c_n]$ is nonempty 
 for all $c_p$ and $c_n$, 
 whenever $\Ucal_p[c_p]$ and $\Ucal_n[c_n]$ are both nonempty. 
 Then, clearly $\z_p=\z_n\in\Ucal_p[c_p]\cap\Ucal_n[c_n]$ is the optimal choice of the
 objective function in \eqref{eqn:RCM-representation}. 
 In $\nu$-SVM with a small $\nu>0$, 
 the reduced convex-hulls satisfy $\Ucal_p\cap\Ucal_n=\emptyset$, and hence, 
 $\widehat{\z}_p=\widehat{\z}_n$ and $\widehat{\w}=\0$ hold. 

The bias term $b$ in the linear decision function is not directly obtained from the
optimal solution of \eqref{eqn:RCM-representation} without knowing 
the explicit form of the loss function $\ell$. 
A simple way of estimating the bias term is to choose
$\widehat{b}=-(\widehat{\w}^T\widehat{\z_p}+\widehat{\w}^T\widehat{\z}_n)/2$, which
provides the decision  
boundary bisecting the line segment connecting $\widehat{\z}_p$ and $\widehat{\z}_n$. 
In the learning algorithm proposed in Section \ref{sec:Learning_Algorithm}, 
the bias term is estimated by minimizing the error rate 
\begin{align}
 \min_{b\in\Rbb}\frac{1}{m}\sum_{i=1}^{m}\,\ind{y_i(\widehat{\w}^T\x_i+b)\leq0}. 
 \label{eqn:min-training_error}
\end{align}
Since the estimated normal vector $\widehat{\w}$ is substituted in the above objective
function, the optimization is tractable.

Based on the argument above, we propose the learning algorithm using uncertainty sets in
Figure \ref{fig:simple_learning_algorithm}. 
It is straightforward to apply the kernel method to the algorithm. 
In order to study statistical properties of the learning algorithm based on uncertainty
sets, we need more elaborate description on the algorithm. Details are presented in
Section \ref{sec:Learning_Algorithm}. 
\begin{figure}[t]
 \begin{center}
 \parbox{0.8\linewidth}{
 \begin{description}
  \item[Learning with uncertainty set:]
  \item[Step 1.] Given training samples, we construct parametrized uncertainty sets $\Ucal_p[c]$ and
	      $\Ucal_n[c]$ in some way. 
  \item[Step 2.] Solve \eqref{eqn:RCM-representation}, and obtain the
	      normal vector by \eqref{eqn:opt_w}. 
  \item[Step 3.] The bias term of the decision function is estimated by
	      \eqref{eqn:min-training_error}. 
 \end{description}
} \end{center}
 \caption{Learning algorithm based on uncertainty set.}
 \label{fig:simple_learning_algorithm}
\end{figure}

We show some examples of uncertainty sets \eqref{eqn:uncertainty-set}
associated with popular loss functions. 
In the following examples, the index sets, $M_{p}$ and $M_{n}$, are defined by 
\eqref{eqn:label_index_set} for the training samples $(\x_1,y_1),\ldots,(\x_m,y_m)$,
and let $m_p$ and $m_n$ be $m_p=|M_p|$ and $m_n=|M_n|$, respectively. 
\begin{example}[$\nu$-SVM]
 As explained above, the problem \eqref{eqn:general-loss} is reduced to $\nu$-SVM by
 defining $\ell(z)=\max\{2z/\nu,0\}$. The conjugate function of $\ell$ is given as 
 \begin{align*}
  \ell^*(\alpha)&=\begin{cases}
	     0,      & \alpha\in[0,2/\nu],\\
	     \infty, & \alpha\not\in[0,2/\nu],\\
	    \end{cases}
 \end{align*}
 and the associated uncertainty set is defined by 
 \begin{align*}
  o\in\{p,n\},\quad
  \Ucal_o[c]&=
  \begin{cases}
   \displaystyle
   \bigg\{
   \sum_{i\in{M_o}}\alpha_i\x_i\,:\,\sum_{i\in{M_o}}\alpha_i=1,\,
   0\leq\alpha_i\leq\frac{2}{m\nu},\,i\in{M_o}
   \bigg\},& c\geq0,\\
   \displaystyle
   \emptyset, & c<0. 
  \end{cases}\quad
 \end{align*}
 For $c\geq0$, the uncertainty set consists of the reduced convex-hull
 of training samples, and it does not depend on the parameter $c$. 
 In addition, the negative $c$ is infeasible. 
 Hence, 
 in the problem \eqref{eqn:RCM-representation}, optimal solutions of $c_p$ and $c_n$
 are given as $c_p=c_n=0$, and the problem is reduced to the simple minimum distance
 problem. 
\end{example}

\begin{example}[Truncated quadratic loss]
\label{example:uncertainty_Truncated_quadratic}
 Now consider $\ell(z)=(\max\{1+z,0\})^2$. The conjugate function is
 \begin{align*}
  \ell^*(\alpha)=
  \begin{cases}
   \displaystyle  -\alpha+\frac{\alpha^2}{4}, & \alpha\geq0,\\
   \displaystyle  \infty, & \alpha<0.
  \end{cases}
 \end{align*}
 For $o\in\{p,n\}$, we define $\bar{\x}_o$ and  $\widehat{\Sigma}_o$ as the empirical mean
 and the empirical covariance  matrix of the samples $\{\x_i\,:\,i\in{M_o}\}$, i.e., 
 \begin{align*}
  \bar{\x}_o=\frac{1}{m_o}\sum_{i\in{M_o}}\x_i,\quad
  \widehat{\Sigma}_o=\frac{1}{m_o}\sum_{i\in{M_o}}(\x_i-\bar{\x}_o)(\x_i-\bar{\x}_o)^T. 
 \end{align*}
 Suppose that $\widehat{\Sigma}_o$ is invertible. 
 Then, the uncertainty set corresponding to the truncated quadratic loss is given as
 \begin{align*}
  o\in\{p,n\},\quad
  \Ucal_o[c]
  &=
   \bigg\{
   \sum_{i\in{M_o}}\alpha_i\x_i\,:\,\sum_{i\in{M_o}}\alpha_i=1,\,
  \alpha_i\geq{0},\,i\in{M_o},\,\sum_{i\in{M_o}}\alpha_i^2\leq \frac{4(c+1)}{m}
  \bigg\}\\
  &=
   \bigg\{
  \z\in\conv\{\x_i:i\in{M_o}\}\,:\,
  (\z-\bar{\x}_o)^T\widehat{\Sigma}_o^{-1}(\z-\bar{\x}_o)\leq\frac{4(c+1)m_o}{m}
  \bigg\}. 
 \end{align*}
 To prove the second equality, let us define the matrix
 $X=(\x_1,\ldots,\x_{m_o})\in\Rbb^{d\times{}m_o}$. 
 For $\alphabold_o=(\alpha_i)_{i\in{M_o}}$ satisfying the constraints, the equality 
 $\z=\sum_{i\in{M_o}}\alpha_i\x_i=(X-\bar{\x}_o\1^T)\alphabold_o+\bar{\x}_o$ holds, 
 where $\1=(1,\ldots,1)^T\in\Rbb^{m_o}$. 
 Then, the singular value decomposition of the matrix $X-\bar{\x}_o\1^T$ and the
 constraint $\|\alphabold_o\|^2\leq4(c+1)/m$ yield the second equality. 
 A similar uncertainty set is used in minimax probability machine (MPM)
 \citep{Lanckriet:2003:RMA:944919.944934} and maximum margin MPM \citep{Nath07}, 
 though the constraint, $\z\in\conv\{\x_i:i\in{M_o}\}$, is not imposed in 
 these learning methods. 
\end{example}

\begin{example}[exponential loss]
 The loss function $\ell(z)=e^z$ is used in Adaboost 
 \citep{FreundSchapire97,Friedman98additivelogistic}. 
 The conjugate function is equal to 
 \begin{align*}
  \ell^*(\alpha)=
  \begin{cases}
   -\alpha+\alpha\log\alpha, & \alpha\geq0,\\
   \infty, & \alpha<0. 
  \end{cases}
 \end{align*}
  Hence, the corresponding uncertainty set is defined as 
  \begin{align*}
   \Ucal_o[c]=
   \bigg\{
   \sum_{i\in{M_o}}\alpha_i\x_i\,:\,\sum_{i\in{M_o}}\alpha_i=1,\,
   \alpha_i\geq{0},\,i\in{M_o},\,
   \sum_{i\in{M_o}}\alpha_i\log\frac{\alpha_i}{1/m_o}\leq{}c+1+\log\frac{m_o}{m}\bigg\}
  \end{align*}
 for $o\in\{p,n\}$. 
 In the uncertainty set, the Kullback-Leibler divergence from the weight
 $\alpha_i,i\in{M_o}$ to the  uniform weight is bounded above. 
\end{example}

In this section, we derived parametrized uncertainty sets associated with convex loss
functions. Inversely, if the uncertainty set is represented as the form of
\eqref{eqn:uncertainty-set}, there exists the corresponding loss function. 
When we consider statistical properties of the classifier estimated based on the
uncertainty set, we can study the equivalent estimator derived from the corresponding loss
function. 
We have many theoretical tools to analyze such estimators. 
However, if the uncertainty set does not have the expression of
\eqref{eqn:uncertainty-set}, the corresponding loss function would not exist. 
In this case, we cannot apply the standard theoretical tools to understand statistical
properties of learning algorithms based on such uncertainty sets. One way to remedy the
drawback is to revise the uncertainty set so as to possess the corresponding loss
function. The next section is devoted to study a way of revising the uncertainty set.

\section{Revision of Uncertainty Sets}
\label{sec:Revision_Uncertainty_Set}
Given a parametrized uncertainty set, generally there does not exist the loss function
which corresponds to the uncertainty set. 
In this section, we present a way of revising the uncertainty set such that 
there exists a corresponding loss function. 

We consider two kinds of representations for parametrized uncertainty sets:
one is vertex representation, and the other is level-set representation. 
Let $M_p$ and $M_n$ be index sets defined in \eqref{eqn:label_index_set}, and we define
$m_p=|M_p|$ and $m_n=|M_n|$. 
For $o\in\{p,n\}$, let $L_o$ be a closed, convex, proper function on $\Rbb^{m_o}$, and 
$L_o^*$ be the conjugate function of $L_o$. 
The argument of $L_o^*$ is represented by $\alphabold_o=(\alpha_i)_{i\in{M_o}}$. 
The \emph{vertex representation} of the uncertainty set is defined as 
\begin{align}
 \Ucal_o[c]=
 \bigg\{
 \sum_{i\in{M_o}}\alpha_i\x_i\,:\,
 L_o^*(\alphabold_o)\leq{c}
 \bigg\},\quad{}o\in\{p,n\}. 
 \label{eqn:uncertainty-vertex-rep}
\end{align}
In Example \ref{example:uncertainty_Truncated_quadratic}, the function 
$L_o^*(\alphabold_o)=\frac{m}{4}\sum_{i\in{M_o}}\alpha_i^2-1$ is employed. 
On the other hand, let us define $h_o:\Rbb^d\rightarrow\Rbb$ as a closed, convex, proper
function, and $h_o^*$ be the conjugate of $h_o$. 
The \emph{level-set representation} of the uncertainty set is defined by 
\begin{align}
 \Ucal_o[c]=
 \bigg\{
 \sum_{i\in{M_o}}\alpha_i\x_i\,:\,
 h_o^*\big(\sum_{i\in{M_o}}\alpha_i\x_i\big)\leq{c}
 \bigg\},\quad{}o\in\{p,n\}. 
 \label{eqn:uncertainty-levelset-rep}
\end{align}
The function $h_o^*$ may depend on the population distribution. 
We suppose that $h_o^*$ does not depend on the sample points, $\x_i,i\in{M_o}$. 
In Example \ref{example:uncertainty_Truncated_quadratic}, 
the second expression of the uncertainty set involves the convex function 
$h_o^*(\z)=(\z-\bar{\x}_o)^T\widehat{\Sigma}_o^{-1}(\z-\bar{\x}_o)$. 
This function does not satisfy the assumption, since $h_o^*$ depends on training samples via
$\bar{\x}_o$ and $\widehat{\Sigma}_o$. 
Instead, the function $h_o^*(\z)=(\z-\mubold_o)^T\Sigma_o^{-1}(\z-\mubold_o)$ 
with the population mean $\mubold_o$ and the population covariance matrix $\Sigma_o$ 
meets the condition. When $\mubold_o$ and $\Sigma_o$ are replaced with the estimated
parameters based on a prior knowledge or a set of samples independent of the training
samples, $\{\x_i:i\in{M_o}\}$, the function $h_o^*$ with the estimated parameters still 
satisfies the condition we imposed above. 


\subsection{From uncertainty sets to loss functions} 
\label{subsec:From_uncertainty_to_loss}
In popular learning algorithms using uncertainty sets such as hard-margin SVM, $\nu$-SVM
and maximum margin MPM, the decision function is estimated 
by solving the minimum distance problem \eqref{eqn:min-distance-prob} with
$\Ucal_p=\Ucal_p[\bar{c}_p]$ and $\Ucal_n=\Ucal_n[\bar{c}_n]$, 
where $\bar{c}_p$ and $\bar{c}_n$ are prespecified
constants. In order to investigate the statistical properties of the learning algorithm
using uncertainty sets, we consider the primal expression of a variant of the minimum 
distance problem \eqref{eqn:min-distance-prob}. 

In Section \ref{sec:Loss_and_Uncertainty}, we derived the problem
\eqref{eqn:RCM-representation} as the dual form of \eqref{eqn:general-loss}. 
Here, we consider the following optimization problem to obtain the loss function 
corresponding to given uncertainty sets having the vertex representation
\eqref{eqn:uncertainty-vertex-rep}, 
\begin{align}
 \begin{array}{l}
  \displaystyle
  \min_{c_p,c_n,\z_p,\z_n} c_p+c_n+\lambda\|\z_p-\z_n\|\\
  \displaystyle
   \st\ c_p,c_n\in\Rbb,\\
  \displaystyle
   \phantom{\st}\ 
   \z_p\in\Ucal_p[c_p]\cap\conv\{\x_i:i\in{M_p}\},\\
  \displaystyle 
   \phantom{\st}\ 
   \z_n\in\Ucal_n[c_n]\cap\conv\{\x_i:i\in{M_n}\}. 
 \end{array}
 \label{eqn:opt-based-on-uncertainty_set}
\end{align}
In the above problem the constraints, $\z_o\in\conv\{\x_i:i\in{M_o}\},o\in\{p,n\}$, 
are added, since the corresponding uncertainty set \eqref{eqn:uncertainty-set}
has the same constraint. 
We derive the primal problem corresponding to 
\eqref{eqn:opt-based-on-uncertainty_set} via the min-max theorem. 
A brief calculation yields that 
 \eqref{eqn:opt-based-on-uncertainty_set} is equivalent to 
\begin{align}
 \begin{array}{l}
  \displaystyle
  \min_{\alphabold}\,
  L_p^*(\alphabold_p)+L_n^*(\alphabold_n)
  +\lambda
  \big\| \sum_{i=1}^m\alpha_i{}y_i\x_i\big\|\\ 
  \displaystyle
   \st\   
   \sum_{i\in{M_p}}\alpha_i=1,\  \sum_{j\in{M_n}}\alpha_j=1,\
    \alpha_i\geq0\ (i=1,\ldots,m). 
 \end{array}
 \label{eqn:loss-expression_opt-based-on-uncertainty_set}
\end{align}
If there is no duality gap, the corresponding primal formulation of
\eqref{eqn:loss-expression_opt-based-on-uncertainty_set} is given as
\begin{align}
 \begin{array}{l}
  \displaystyle
  \inf_{\w,b,\rho,\xibold_p,\xibold_n}-2\rho
  +L_p(\xibold_p)+L_n(\xibold_n),\\
  \displaystyle
  \st\,
 \rho-y_i(\w^T\x_i+b)\leq \xi_i,\,i=1,\ldots,m,
 \quad \|\w\|^2\leq\lambda^2, 
 \end{array}
 \label{eqn:loss-expression_g^*}
\end{align}
where $\xibold_o$ is defined as $\xibold_o=(\xi_i)_{i\in{M_o}}$ for $o\in\{p,n\}$. 

In the primal expression \eqref{eqn:loss-expression_g^*}, 
$L_p$ and $L_n$ are regarded as the loss function for the decision function $\w^T\x+b$ 
on training samples. 
In general, however, the loss function is not represented as the empirical mean over
training samples. Thus, we cannot apply the standard theoretical tools to investigate
statistical properties 
such as Bayes risk consistency for the learning algorithm based on
\eqref{eqn:opt-based-on-uncertainty_set} or \eqref{eqn:loss-expression_g^*}. 
On the other hand, if the problem \eqref{eqn:loss-expression_g^*} is described as the
empirical loss minimization, we can study statistical properties of the algorithm 
by applying the statistical theory developed by
\cite{Vapnik98,steinwart05:_consis_of_suppor_vector_machin,bartlett06:_convex_class_risk_bound}. 
To link the uncertainty set approach with the empirical loss minimization, 
we consider a revision of the uncertainty set.

\subsection{Revised uncertainty sets and corresponding loss functions}
\label{subsec:Revised_uncertainty_set_and_corresponding_loss}
We propose a way of revising uncertainty sets such that 
the primal form \eqref{eqn:loss-expression_g^*} is represented as minimization of the
empirical mean of a loss function. 
Remember that the additivity of the
function is kept unchanged in the conjugate function, i.e., 
$(\ell_1(z_1)+\ell_2(z_2))^*=(\ell_1(z_1))^*+(\ell_2(z_2))^*$. 
\begin{description}
 \item[Revision of uncertainty set defined by vertex representation: ]  
	    Suppose that the uncertainty set is described by 
	    \eqref{eqn:uncertainty-vertex-rep}. 
	    For $o\in\{p,n\}$, we define $m_o$-dimensional vectors $\1_o=(1,\ldots,1)$ and
	    $\0_o=(0,\ldots,0)$. 
	    For the convex function $L_o^*:\Rbb^{m_o}\rightarrow\Rbb$, we define
	    $\bar{\ell}^*:\Rbb\rightarrow\Rbb\cup\{\infty\}$ by 
	    \begin{align}
	     \bar{\ell}^*(\alpha)
	     =
	     \begin{cases}
	      \displaystyle
	      L_p^*(\frac{\alpha}{m}\1_p)+L_n^*(\frac{\alpha}{m}\1_n)
	      -L_p^*(\0_p)-L_n^*(\0_n)
	      &\alpha\geq0,\\
	      \displaystyle
	      \infty, & \alpha<0. 
	     \end{cases}
	     \label{eqn:revised-uncertaintyset-diagonal}
	    \end{align}
	    The revised uncertainty set $\bar{\Ucal}_o[c],\,o\in\{p,n\}$ is defined as 
	    \begin{align*}
	     \bar{\Ucal}_o[c]=
	     \bigg\{
	     \sum_{i\in{M_o}}\alpha_i\x_i\,:\,
	     \sum_{i\in{M_o}}\alpha_i=1,\,\alpha_i\geq 0,\,i\in{M_o},\,
	     \frac{1}{m}\sum_{i\in{M_o}}\bar{\ell}^*(\alpha_im)\leq{c}
	     \bigg\}. 
	    \end{align*}


 \item[Revision of uncertainty set defined by level-set representation: ]  
	    Suppose that the uncertainty set is described by 
	    \eqref{eqn:uncertainty-levelset-rep} and that 
	    the mean of the input vector $\x$ conditioned on the 
	    positive (resp. negative) label is given as 
	    $\mubold_p\,(\text{resp.}\,\mubold_n)$. 
	    The null vector is denoted as $\0$. 
	    We define the function $\bar{\ell}^*:\Rbb\rightarrow\Rbb$ by 
	    \begin{align}
	     \bar{\ell}^*(\alpha)
	     =
	     \begin{cases}
	      \displaystyle
	      h_p^*(\alpha\frac{m_p}{m}\mubold_p)+h_n^*(\alpha\frac{m_n}{m}\mubold_n)
	      -h_p^*(\0)-h_n^*(\0)
	      &\alpha\geq0,\\
	      \displaystyle
	      \infty, & \alpha<0. 
	     \end{cases}
	     \label{eqn:revised-uncertaintyset-levelset}
	    \end{align}
	    The revised uncertainty set $\bar{\Ucal}_o[c], o\in\{p,n\}$ is defined as 
	    \begin{align*}
	     \bar{\Ucal}_o[c]=\bigg\{
	     \sum_{i\in{M_o}}\alpha_i\x_i\,:\,
	     \sum_{i\in{M_o}}\alpha_i=1,\,\alpha_i\geq 0,\,i\in{M_o},\,
	     \frac{1}{m}\sum_{i\in{M_o}}\bar{\ell}^*(\alpha_im)\leq{c},\,
	     \bigg\}. 
	    \end{align*}
	    We apply the parallel shift of training samples 
	    so as to be $\mubold_p\neq\0$ or $\mubold_n\neq\0$. 
\end{description}


We explain the reason why the revised uncertainty set is defined as above. 
In the revision \eqref{eqn:revised-uncertaintyset-diagonal}, 
the uncertainty set is kept unchanged, when the function $L_p^*+L_n^*$ is described in the
additive form. The precise description is presented in the following theorem. 
\begin{theorem}
 \label{theorem:conservation_law}
 Let $L_o^*:\Rbb^{m_o}\rightarrow\Rbb,\,o\in\{p,n\}$ be convex functions, 
 and $\bar{\ell}^*$ be the function defined by \eqref{eqn:revised-uncertaintyset-diagonal}
 for given $L_p^*$ and $L_n^*$. 
 Suppose that $\ell:\Rbb\rightarrow\Rbb\cup\{\infty\}$ is 
 a closed, convex, proper function such that 
 $\ell^*(0)=0$ and $\ell^*(\alpha)=\infty$ for $\alpha<0$ hold. 
 \begin{enumerate}
  \item Suppose that the equality 
	\begin{align*}
	 L_p^*(\alphabold_p)+L_n^*(\alphabold_n)
	 -L_p^*(\0_p)-L_n^*(\0_n)
	 =\frac{1}{m}\sum_{i=1}^m\ell^*(\alpha_im)
	\end{align*}
	holds for all non-negative $\alpha_i,\,i=1,\ldots,m$. 
	Then, the equality $\bar{\ell}^*=\ell^*$ holds. 
  \item Suppose that the equality 
	\begin{align*}
	 L_p^*(\alpha\1_p)+L_n^*(\alpha\1_n)
	 -L_p^*(\0_p)-L_n^*(\0_n)
	 =\frac{1}{m}\sum_{i=1}^m\ell^*(\alpha{}m)
	 =\ell^*(\alpha{}m)
	\end{align*}
	holds for all $\alpha\geq0$. Then, the equality $\bar{\ell}^*=\ell^*$ holds. 
 \end{enumerate}
\end{theorem}
\begin{proof}
 We prove the first statement. From the definition of $\bar{\ell}^*$ and the assumption on
 $\ell^*$, the equality $\ell^*(\alpha)=\bar{\ell}^*(\alpha)$ holds for
 $\alpha<0$. Suppose $\alpha\geq{0}$. 
 The assumption on $L_p^*$ and $L_n^*$ leads to 
 $L_p^*(\frac{\alpha}{m}\1_p)+L_n^*(\frac{\alpha}{m}\1_n)-L_p^*(\0_p)-L_n^*(\0_n)
 =\ell^*(\alpha)$. 
 Hence, we have $\ell^*=\bar{\ell}^*$. 
 The second statement of the theorem is straightforward. 
\end{proof}
Theorem \ref{theorem:conservation_law} implies that 
the transformation of 
$L_p^*+L_n^*$ to $\frac{1}{m}\sum_{i=1}^m\bar{\ell}^*(\alpha_i{}m)$ 
is a projection onto the set of functions with the additive form. 
In addition, the second statement of Theorem \ref{theorem:conservation_law} 
denotes that the projection is uniquely determined when we impose the condition 
that the values on the diagonal 
$\{(\alpha,\ldots,\alpha)\in\Rbb^{m}\,:\,\alpha\geq0\}$ 
are unchanged. 

Next, we explain the validity of the formula \eqref{eqn:revised-uncertaintyset-levelset}. 
We want to find a function $\bar{\ell}^*(\alpha)$ such that
$h_p^*(\sum_{i\in{M_p}}\alpha_i\x_i)+h_n^*(\sum_{i\in{M_n}}\alpha_i\x_i)-h_p^*(\0)-h_n^*(\0)$ 
is close to 
$\frac{1}{m}\sum_{i=1}^m\bar{\ell}^*(m\alpha_i)$ in some sense. 
We substitute 
$\alpha_i=\alpha/m$ into $h_o^*(\sum_{i\in{M_o}}\alpha_i\x_i),\,o\in\{p,n\}$. 
In the large sample limit,
$h_o^*(\sum_{i\in{M_o}}\alpha/m\,\x_i)$ is approximated by 
$h_o^*(\alpha\frac{m_o}{m}\mubold_o)$. 
Suppose that 
\[
h_p^*(\alpha\frac{m_p}{m}\mubold_p)+h_n^*(\alpha\frac{m_n}{m}\mubold_n)-h_p^*(\0)-h_n^*(\0)
\]
is
represented as $\frac{1}{m}\sum_{i=1}^m\bar{\ell}^*(\frac{\alpha}{m}m)=\bar{\ell}^*(\alpha)$. 
Then, we obtain \eqref{eqn:revised-uncertaintyset-levelset}. 

For the revised uncertainty sets $\bar{\Ucal}_p[c]$ and $\bar{\Ucal}_n[c]$, 
the corresponding primal problem of 
\begin{align}
 \min_{c_p,c_n,\z_p,\z_n}c_p+c_n+\lambda\|\z_p-\z_n\|\quad
 \st \z_p\in\bar{\Ucal}_p[c_p],\ \z_n\in\bar{\Ucal}_n[c_n]
 \label{eqn:revised-uncertainty-problem}
\end{align}
is given as 
\begin{align*}
 &\inf_{w,b,\rho,\xi_p,\xi_n}
 -2\rho
 +\frac{1}{m}\sum_{i=1}^m\bar{\ell}(\xi_i)
 &\st\,\rho-y_i(\w^T\x_i+b)\leq\xi_i,\,i=1,\ldots,m,\ \ 
 \|\w\|^2\leq\lambda^2. 
\end{align*}
The revision of the uncertainty sets leads to the empirical mean of the revised loss
function $\bar{\ell}$. 
When we study statistical properties of the estimator given by the optimal 
solution of \eqref{eqn:revised-uncertainty-problem}, we can apply the standard theoretical
tools, since the objective in the primal expression is described by the empirical mean of 
the revised loss functions. 

We show some examples to illustrate how the revision of the uncertainty set works. 
\begin{example}
\label{example:quad-uncertainty-set-to-loss}
 Let $L_o^*,\,o\in\{p,n\}$ be the convex function $L_o^*(\alphabold_o)=\alphabold_o^TC_o\alphabold_o$, 
 where $C_o$ is a positive definite matrix. 
 The revised function defined by \eqref{eqn:revised-uncertaintyset-diagonal} 
 is given as 
 \begin{align*}
  \bar{\ell}^*(\alpha)=\alpha^2\frac{\1_p^TC_p\1_p+\1_n^TC_n\1_n}{m^2}
 \end{align*}
 for $\alpha\geq0$. 
 Then, we have 
 \begin{align*}
  \frac{1}{m}\sum_{i=1}^m\bar{\ell}^*(\alpha_im)
  =
  \frac{\1_p^TC_p\1_p+\1_n^TC_n\1_n}{m}\sum_{i=1}^{m}\alpha_i^2
 \end{align*}
 When both $C_p$ and $C_n$ are the identity matrix, 
 the equality 
 \begin{align*}
  L_p^*(\alphabold_p)+L_n^*(\alphabold_n)=\frac{1}{m}\sum_{i=1}^{m}\bar{\ell}^*(\alpha_im)
  =\sum_{i=1}^{m}\alpha_i^2
 \end{align*}
 holds. Let $k$ be $k=\1_p^TC_p\1_p+\1_n^TC_n\1_n$. 
 Then, the revised uncertainty set is given as 
 \begin{align*}
  o\in\{p,n\},\quad
  \bar{\Ucal}_o[c]
  &=
  \bigg\{
  \sum_{i\in{M_o}}\alpha_i\x_i\,:
  \sum_{i\in{M_o}}\alpha_i=1,\,\alpha_i\geq0\,(i\in{M_o}), 
  \sum_{i\in{M_o}}\alpha_i^2 \leq  \frac{cm}{k}
  \bigg\}. 
 \end{align*}
 For $o\in\{p,n\}$, let $\bar{\x}_o$ and $\widehat{\Sigma}_o$ be the empirical mean and
 the empirical covariance matrix, 
 \begin{align*}
  \bar{\x}_o=\frac{1}{m_o}\sum_{i\in{M_o}}\x_i,\quad
  \widehat{\Sigma}_o
  =\frac{1}{m_o}\sum_{i\in{M_o}}(\x_i-\bar{\x}_o)(\x_i-\bar{\x}_o)^T. 
 \end{align*}
 If $\widehat{\Sigma}_o$ is invertible, we have 
 \begin{align*}
  \bar{\Ucal}_o[c]
  =
  \bigg\{
  \z\in\conv\{\x_i:i\in{M_o}\}\,:\,
  (\z-\bar{\x}_o)^T\widehat{\Sigma}_o^{-1}(\z-\bar{\x}_o)\leq
  \frac{c{}mm_o}{k}
  \bigg\}. 
 \end{align*}
 In the learning algorithm based on the revised uncertainty set, 
 the estimator is obtained by solving 
 \begin{align*}
  &\phantom{\Longleftrightarrow}
  \min_{c_p,c_n,\z_p,\z_n}c_p+c_n+\lambda\|\z_p-\z_n\|\ \ \st\ 
  \z_p\in\bar{\Ucal}_p[c_p],\,\z_n\in\bar{\Ucal}_n[c_n] \\
  &\Longleftrightarrow
  \min_{c_p,c_n,\z_p,\z_n} c_p+c_n+\frac{m^2\lambda}{4k}\|\z_p-\z_n\|\ \ \st\ 
  \z_p\in\bar{\Ucal}_p\bigg[\frac{4c_p k}{m^2}\bigg],\,
  \z_n\in\bar{\Ucal}_n\bigg[\frac{4c_n k}{m^2}\bigg]. 
 \end{align*}
 The corresponding primal expression is given as 
 \begin{align*}
  \min_{\w,b,\rho,\xibold}-2\rho
  +\frac{1}{m}\sum_{i\in{M_p}}\xi_i^2\ \  
  \st\ \rho-y_i(\w^T\x_i+b)\leq\xi_i,\,0\leq\xi_i,\,\forall{i},\ 
  \|\w\|^2\leq{}\left(\frac{m^2\lambda}{4k}\right)^2. 
 \end{align*}
\end{example}

\begin{example}
\label{example:quad-levelset-uncertainty-set-to-loss}
 We define $h_o^*:\Xcal\rightarrow\Rbb$ for $o\in\{p,n\}$ by 
 \begin{align*} 
  h_o^*(\z)=(\z-\mubold_o)^T C_o(\z-\mubold_o) 
 \end{align*}
 where $\mubold_o$ is the mean vector of the input vector $\x$ conditioned on each label 
 and $C_o$ is a positive definite matrix. 
 In practice, the mean vector is estimated by using a prior knowledge 
 which is independent of the training samples $\{(\x_i,y_i):i=1,\ldots,m\}$. 
 Suppose that $\mubold_o\neq{\0}$. 
 Then, for $\alpha\geq0$, the revision of \eqref{eqn:revised-uncertaintyset-levelset} leads to 
 \begin{align*}
  \bar{\ell}^*(\alpha)
  &=
   \left((\alpha\frac{m_p}{m}-1)^2-1\right)\mubold_p^TC_p\mubold_p
  +\left((\alpha\frac{m_n}{m}-1)^2-1\right)\mubold_n^TC_n\mubold_n\\
  &=  
  b_1\alpha+b_2\alpha^2, 
 \end{align*}
 where $b_1$ and $b_2(>0)$ are constant numbers. 
 Thus, we have
 \begin{align*}
  \bar{\Ucal}_o[c]
  &=
  \bigg\{
  \sum_{i\in{M_o}}\alpha_i\x_i \,:\,
  \sum_{i\in{M_o}}\alpha_i=1,\,\alpha_i\geq0\, (i\in{M_o}),\,
  \sum_{i\in{M_o}}\alpha_i^2\leq
  \frac{c-b_1}{mb_2}
  \bigg\}\\
  &=
  \bigg\{
  \z\in\conv\{\x_i:i\in{M_o}\} \,:\,
  (\z-\bar{\x}_o)^T\widehat{\Sigma}_o^{-1}(\z-\bar{\x}_o)\leq
  m_o\cdot\frac{c-b_1}{mb_2}
  \bigg\}, 
 \end{align*}
 where $\bar{\x}_o$ and $\widehat{\Sigma}_o$ are the estimators of 
 the mean vector and the covariance matrix 
 based on training samples $\{\x_i:i\in{M_o}\}$. 
 The corresponding loss function is obtained in the same way as
 Example \ref{example:quad-uncertainty-set-to-loss}. 
 Figure \ref{fig:revised-ellipsoidal_uncertainty-set} illustrates an example of the
 revision of the uncertainty set. 
 In the left panel, the uncertainty set does not match the distribution of the training
 samples. The revised uncertainty set in the right panel seems to well approximate 
 the dispersal of the training samples. 
 \begin{figure}
  \centering
   \begin{tabular}{cc}
   \includegraphics[scale=0.4]{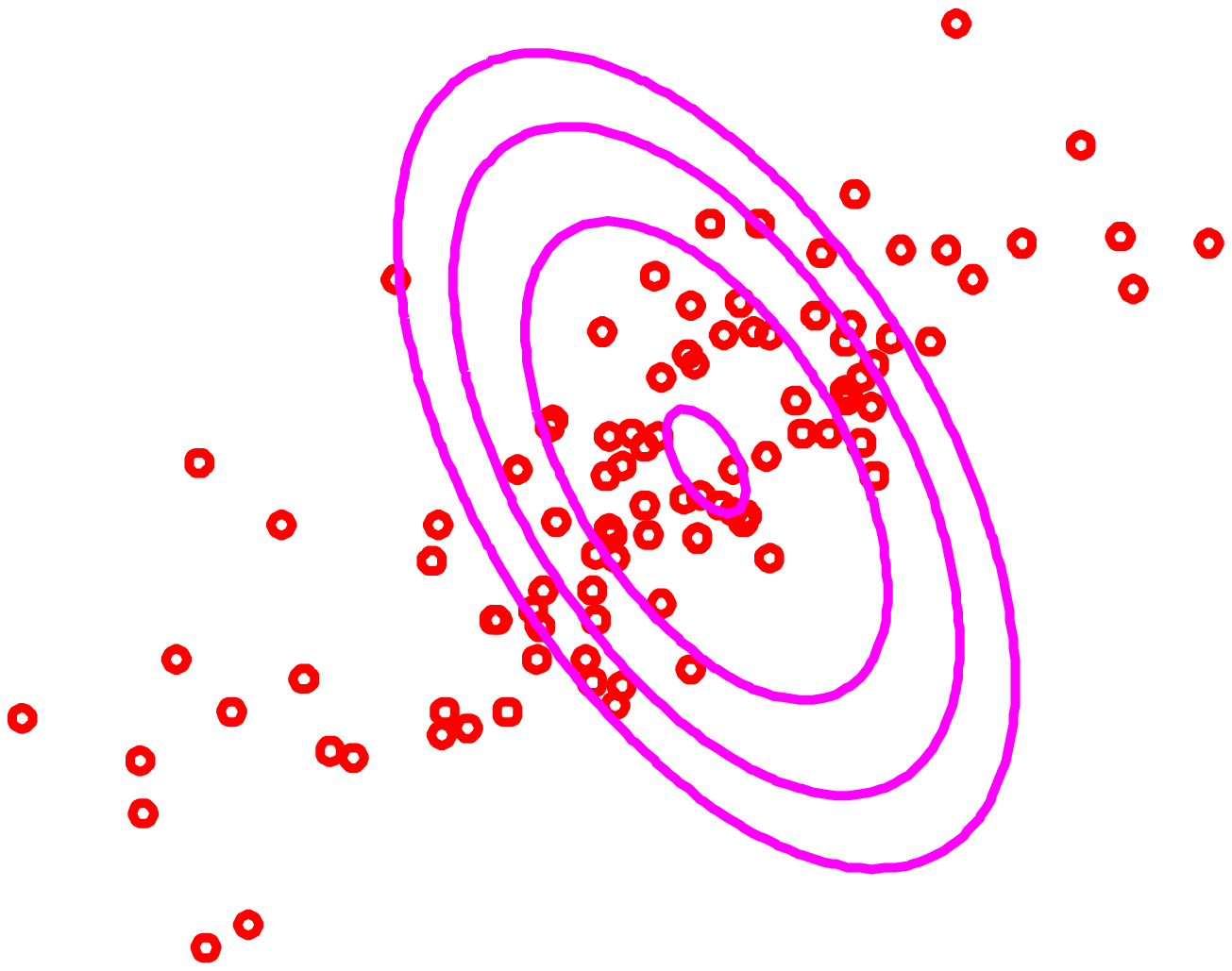} &
   \includegraphics[scale=0.4]{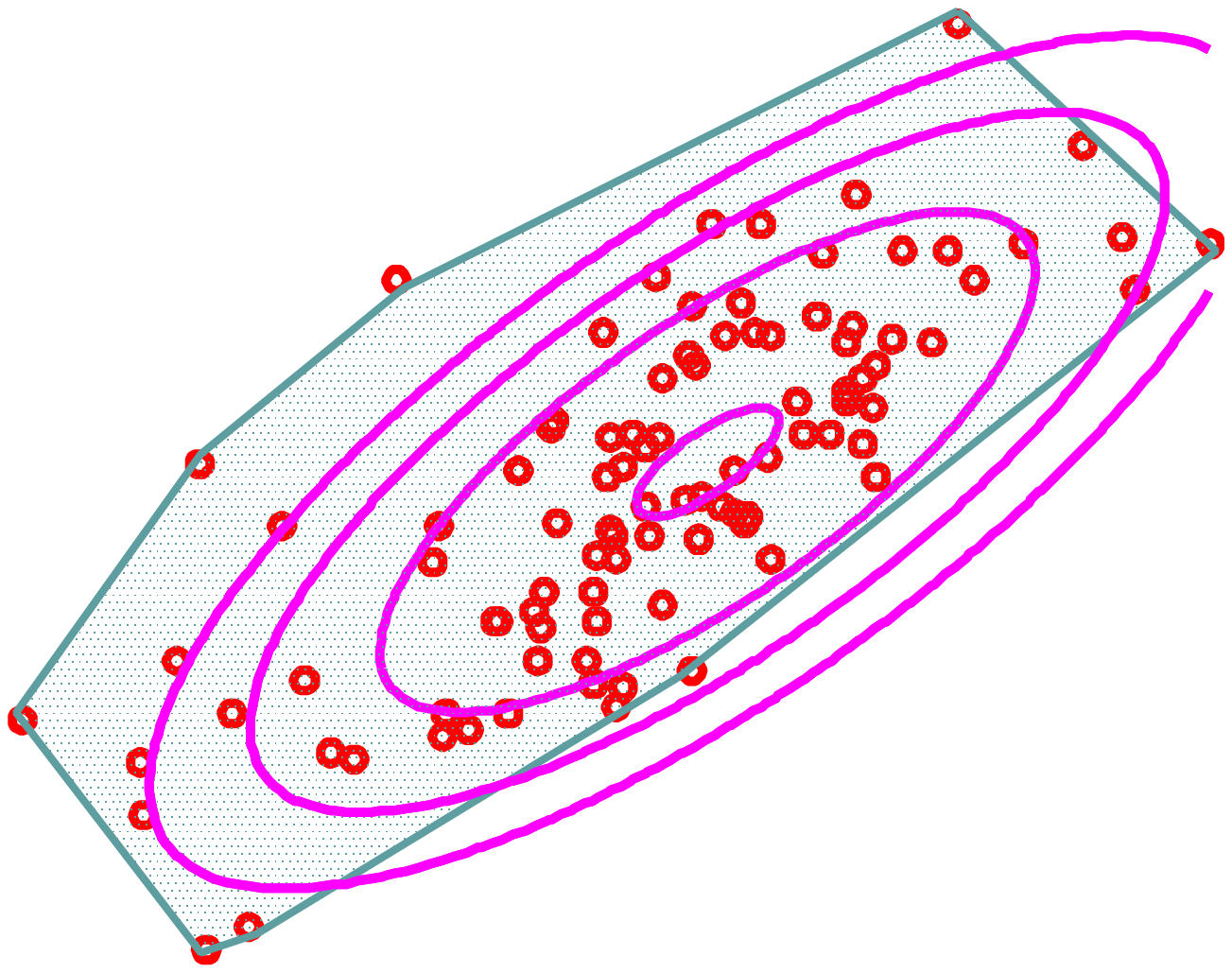}\\
     original uncertainty set $\Ucal_p[c]$\ \  & \ \ 
	 revised uncertainty set $\bar{\Ucal}_p[c]$
   \end{tabular}
  \caption{
  Training samples and the uncertainty sets are depicted. 
  Left panel: the original uncertainty set for the positive label. 
  Right panel: the revised uncertainty set which consists of the intersection of the
  ellipsoid and the convex-hull of the input vectors with positive label. 
  }
  \label{fig:revised-ellipsoidal_uncertainty-set}
 \end{figure}
\end{example}

\begin{example}
 \label{exam:uncertaity-set_with_error}
 We suppose that for $o\in\{p,n\}$, $\mubold_o$ is the mean vector and $\Sigma_o$ is the covariance matrix of the input
 vector conditioned on each label. 
 We define the uncertainty set by 
 \begin{align*}
  o\in\{p,n\},\quad
  \Ucal_o[c]=
  \left\{ z\in\conv\{\x_i:i\in{M_o}\}
  \,:\,(\z-\mubold)^T \Sigma_o^{-1}(\z-\mubold)\leq{c},\,\forall\mubold\in{}\mathcal{A}
  \right\}, 
 \end{align*}
 where $\mathcal{A}$ denotes the estimation error of the mean vector $\mubold$. 
 For a fixed radius $r>0$, $\mathcal{A}$ is defined as 
 \begin{align*}
  \mathcal{A}=\left\{
  \mubold\in\Xcal~:~
  (\mubold-\mubold_o)^T\Sigma_o^{-1}(\mubold-\mubold_o)\leq{r^2}
  \right\}. 
 \end{align*}
 The uncertainty set with estimation error is used by
 \cite{Lanckriet:2003:RMA:944919.944934} in MPM. 
 The above uncertainty sets will be useful, when the probability in the training  phase is
 slightly different from that in the test phase. 
 Brief calculation yields that $\Ucal_o[c]$ is
 represented by the level set of the convex function
 \begin{align*}
  h_o^*(\z)
  =
  \max_{\mubold\in{}\mathcal{A}}\, 
  (\z-\mubold)^T \Sigma_o^{-1}(\z-\mubold)
  =\bigg(\sqrt{(\z-\mubold_o)^T\Sigma_o^{-1}(\z-\mubold_o)}+r\bigg)^2 
 \end{align*}
 The revised uncertainty set $\bar{\Ucal}_o[c]$ is defined by the function
 $\bar{\ell}^*$ which is given as 
 \begin{align}
  \bar{\ell}^*(\alpha)
  &=
  \left(\bigg|\alpha\frac{m_p}{m}-1\bigg|\sqrt{\mubold_p^T\Sigma_p^{-1}\mubold_p}+r\right)^2
  -
  \left(\sqrt{\mubold_p^T\Sigma_p^{-1}\mubold_p}+r\right)^2\nonumber\\
  &\phantom{=}
  +\left(\bigg|\alpha\frac{m_n}{m}-1\bigg|\sqrt{\mubold_n^T\Sigma_n^{-1}\mubold_n}+r\right)^2
  -
  \left(\sqrt{\mubold_n^T\Sigma_n^{-1}\mubold_n}+r\right)^2. 
  \label{eqn:loss-for-uncertainty-set-with-error}
 \end{align}
 We suppose that $\mubold_p\neq\0$ and $\mubold_n=\0$ hold. 
 Let $d=\sqrt{\mubold_p^T\Sigma_p^{-1}\mubold_p}$ and $h=r/d(>0)$. 
 Then, the corresponding loss function is given as 
\begin{align*}
 \bar{\ell}(z)=\frac{md^2}{m_p}\,u\big(\frac{z}{d^2}\big),  
\end{align*}
 where $u(z)$ as defined as 
 \begin{align}
  u(z)=\begin{cases}
	\displaystyle	\ \ 0, & z\leq -2h-2,\\
	\displaystyle	\ \big(\frac{z}{2}+1+h\big)^2, & -2h-2\leq{z}\leq-2h, \\
	\displaystyle	\ z+2h+1, & -2h\leq{z}\leq{2h},\\
	\displaystyle	\ \frac{z^2}{4}+z(1-h)+(1+h)^2, & 2h\leq{z}. 
       \end{cases}
  \label{eqn:uo_for_uncertaintyset_est_err}
 \end{align}
 Figure \ref{fig:revised_loss}
 depicts the function $u(z)$ with $h=1$. 
 When $r=0$ holds, $\bar{\ell}(z)$ is reduced to the truncated quadratic function 
 shown in Example 
 \ref{example:quad-uncertainty-set-to-loss} and
 \ref{example:quad-levelset-uncertainty-set-to-loss}. 
 For positive $r$,
 $\bar{\ell}(z)$ is linear around $z=0$. 
 This implies that by introducing 
 the confidence set of the mean vector, $\mathcal{A}$, 
 the penalty for the misclassification is reduced from quadratic to linear around the
 decision boundary, though the original uncertainty set $\Ucal_o[c]$ does not correspond
 to minimization of an empirical loss function. 
 \begin{figure}[tb]
 \begin{center}
  \includegraphics[scale=0.5]{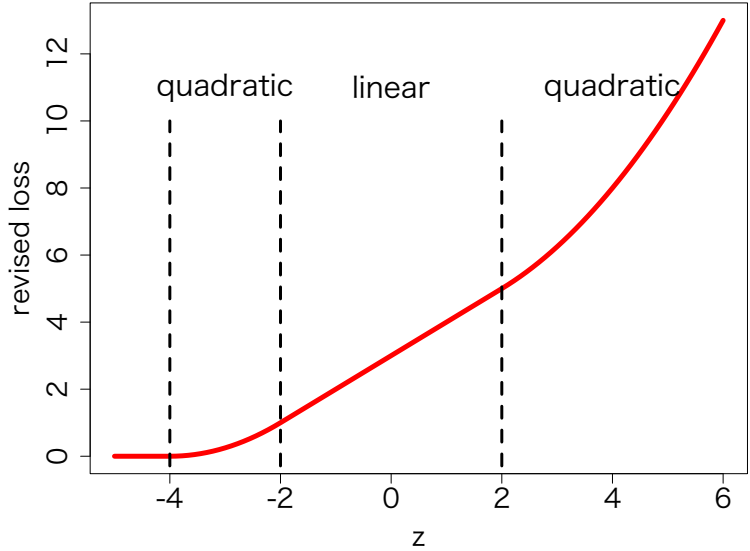}   
 \end{center}
  \caption{The loss function $u(z)$ in Example \ref{exam:uncertaity-set_with_error} is depicted, 
  which corresponds to the revised uncertainty set with the estimation error. }
  \label{fig:revised_loss}
 \end{figure}
\end{example}

\section{Kernel-based Learning Algorithm}
\label{sec:Learning_Algorithm}
We present a kernel variant of the learning algorithm using uncertainty sets. 
Suppose that training samples $(x_1,y_1),\ldots,(x_m,y_m)\in\Xcal\times\Ybin$ are
observed, where $\Xcal$ is not necessarily a linear space. We define the kernel function
$k:\Xcal^2\rightarrow\Rbb$, and let $\Hcal$ be the reproducing kernel Hilbert space (RKHS)
endowed with the kernel function $k$. See \cite{book:Schoelkopf+Smola:2002} for the
details of the kernel estimators in machine learning. 
We consider the estimator of the decision function having the form of $f(x)+b$, where 
$f\in\Hcal,\,b\in\Rbb$. In our algorithm, the function part $f(x)$ and the bias term $b$
are separately estimated. 


Figure \ref{fig:Kernel-based_learning_algorithm} shows 
a kernel variant of the learning algorithm based on uncertainty sets. 
The algorithm is regarded as an extension of $\nu$-SVM and maximum margin MPM, since the
uncertainty set is extended from reduced convex-hull or ellipsoidal uncertainty set 
to general uncertainty set. 
The proposed algorithm is also a revision of the existing method based on the simple
minimum distance problem. We shall illustrate the proposed algorithm in the below. 


In the learning algorithm, 
training samples are divided into two disjoint subsets, $T_1$ and $T_2$, 
which are described as 
\begin{align*}
 T_k&=\{(x_i^{(k)},y_i^{(k)})\,:\,i=1,\ldots,m_k\},\ \ k=1,2. 
\end{align*}
The reason that we decompose the training samples is
to simplify the analysis of statistical properties of the learning algorithm. 
In the kernel-based algorithm, the uncertainty sets, $\Ucal_p[c]$ and $\Ucal_n[c]$, are
convex subsets in $\Hcal$. 
Let $M_p$ and $M_n$ be the index sets of $T_1$ defined by 
\begin{align*}
M_p=\{i\,:\,y_i^{(1)}=+1,\,i=1,\ldots,m_1\},\quad
M_n=\{i\,:\,y_i^{(1)}=-1,\,i=1,\ldots,m_1\}. 
\end{align*}
For $o\in\{p,n\}$, the uncertainty set $\Ucal_o[c]\subset\Hcal$ is defined as a convex
subset of the convex-hull of $\{k(\cdot,x_i^{(1)}):i\in{M_o}\}$. 
Moreover, we assume that the monotonicity $\Ucal_o[c]\subset\Ucal_o[c']$ holds for $c\leq{c'}$.  
If necessary, we revise the uncertainty set as shown in Section
\ref{sec:Revision_Uncertainty_Set} 
in order to link the uncertainty set with a loss function. 

\begin{figure}[p]
\begin{center}
\parbox{0.8\linewidth}{
  \begin{description}
  \item[Inputs. ] Decompose the training samples into two disjoint subsets, 
	     \begin{align*}
	      T_1=\{(x_i^{(1)},y_i^{(1)})\,:\,i=1,\ldots,m_1\},\ 
	      T_2&=\{(x_i^{(2)},y_i^{(2)})\,:\,i=1,\ldots,m_2\}. 
	     \end{align*}
	     For the set of training samples $T_1$, let $M_p$ and $M_n$ be the index sets
	     defined by 
	     $M_p=\{i\,:\,y_i^{(1)}=+1,\,i=1,\ldots,m_1\}$ and
	     $M_n=\{i\,:\,y_i^{(1)}=-1,\,i=1,\ldots,m_1\}$, respectively. 
  \item[Initialization. ] 
	     We define the RKHS $\Hcal$ with the kernel function $k(x,x')$. 
	     Prepare the parametrized uncertainty sets 
	     $\Ucal_p[c]$ and $\Ucal_n[c]$ in $\Hcal$ such that
	     \begin{align*}
	      \Ucal_p[c]\subset\conv\{k(\cdot,x_i^{(1)}):i\in{M_p}\},\quad
	      \Ucal_n[c]\subset\conv\{k(\cdot,x_i^{(1)}):i\in{M_n}\}. 
	     \end{align*}
	     When the uncertainty sets involve some parameters to be estimated, 
	     a prior knowledge or additional samples independent of the training samples $T_1\cup{}T_2$
	     are used for its estimation. 
	     If necessary, we apply the revision of the uncertainty sets presented in
	     Section \ref{sec:Revision_Uncertainty_Set} in order to link the uncertainty
	     set with a loss function. Set the regularization parameter $\lambda>0$. 
  \item[Step 1. ] 
	     Solve the optimization problem, 
	     \begin{align*}
	      \begin{array}{l}
	       \displaystyle
	       \inf_{c_p,c_n,f_p,f_n}\ c_p+c_n+\lambda\|f_p-f_n\|_{\Hcal}\\
	       \displaystyle
		\quad \st\ f_p\in\Ucal_p[c_p],\ f_n\in\Ucal_n[c_n],\ c_p,c_n\in\Rbb. 
	      \end{array}
	     \end{align*}
	     Optimal solutions of $f_p$ and $f_n$ are denoted as $\widehat{f}_p$ and $\widehat{f}_n$. 
	     Define $\widehat{f}$ by 
	     \begin{align*}
	      \widehat{f}=
	      \begin{cases}
	       \displaystyle
	       \frac{\lambda}{\|\widehat{f}_p-\widehat{f}_n\|_{\Hcal}}(\widehat{f}_p-\widehat{f}_n), 
	       & \widehat{f}_p\neq\widehat{f}_n,\\
	       \displaystyle{} 0, & \widehat{f}_p=\widehat{f}_n. 
	      \end{cases}
	     \end{align*}
  \item[Step 2. ] 
	     Solve the one-dimensional optimization problem defined from the estimator
	     $\widehat{f}$ and the data set $T_2$, 
	     \begin{align*}
	      \min_{b\in\Rbb}\,\widehat{\Ecal}_{T_2}(\widehat{f}+b)
	     \end{align*}
	     The optimal solution is denoted as $\widetilde{b}$. 
  \item[Output. ]
	     The estimator of the decision function is given by
	     $\widehat{f}(x)+\widetilde{b}$. 
  \end{description}
}
\end{center}
 \caption{Kernel-based learning algorithm using uncertainty sets.}
 \label{fig:Kernel-based_learning_algorithm}
\end{figure}

When the uncertainty sets involve some parameters to be estimated, 
a prior knowledge or additional samples independent of the training samples $T_1\cup{}T_2$
are used for its estimation. 
For example, the uncertainty set defined by the level set of 
$h_o(\z)=(\z-\mubold_o)^T\Sigma_o^{-1}(\z-\mubold_o),o\in\{p,n\}$ 
involves the mean vector $\mubold_o$ and the covariance matrix $\Sigma_o$. 
In our algorithm, we need to prepare additional samples to estimate $\mubold_o$ and
$\Sigma_o$. 

The subset $T_1$ is used for the estimation of the function part $f\in\Hcal$ 
in the decision function. First, we solve the problem,
\begin{align}
 \begin{array}{l}
  \displaystyle
   \inf_{c_p,c_n,f_p,f_n}\ c_p+c_n+\lambda\|f_p-f_n\|_{\Hcal}\\
  \displaystyle
   \quad \st\ f_p\in\Ucal_p[c_p],\ f_n\in\Ucal_n[c_n],\ c_p,c_n\in\Rbb. 
 \end{array}
 \label{eqn:opt_prob_kernel_estimation_RCM}
\end{align}
Let $\widehat{f}_p$ and $\widehat{f}_n$ be optimal solutions of $f_p$ and $f_n$ in 
\eqref{eqn:opt_prob_kernel_estimation_RCM}. 
Then, in the same way as \eqref{eqn:opt_w}, the function part of the decision function is
estimated by 
\begin{align}
 \widehat{f}=
 \begin{cases}
  \displaystyle
  \frac{\lambda}{\|\widehat{f}_p-\widehat{f}_n\|_{\Hcal}}(\widehat{f}_p-\widehat{f}_n), 
  & \widehat{f}_p\neq\widehat{f}_n,\\
  \displaystyle{} 0, & \widehat{f}_p=\widehat{f}_n. 
 \end{cases}
 \label{eqn:f_estimator}
\end{align}
For the estimation of the bias term $b$, the data set $T_2$ is used. 
The bias estimator $\widetilde{b}$ is an optimal solution of 
\begin{align}
  \min_{b\in\Rbb}\,\widehat{\Ecal}_{T_2}(\widehat{f}+b). 
 \label{eqn:opt_prob_bias_estimation}
\end{align}
Our purpose is to obtain the decision function with a low prediction error. 
Hence, the error rate \eqref{eqn:opt_prob_bias_estimation} is an appropriate criterion for
the estimation of the bias term. 
Though generally the minimization of the training error rate is hard task, the
one-dimensional optimization is easily conducted. 
Then, the estimator of the decision function is given by $\widehat{f}(x)+\widetilde{b}$. 
By separating the training data used in Step 1 and Step 2, we can simplify the statistical
analysis of the estimator.

\section{Statistical Properties of Kernel-based Learning Algorithm}
\label{sec:Statistical_Properties}
In this section, we study statistical properties of the learning algorithm presented in
Figure \ref{fig:Kernel-based_learning_algorithm}. 
Especially, we prove that the expected 0-1 loss of the estimator,
$\Ecal(\widehat{f}+\widetilde{b})$, converges to the Bayes risk $\Ecal^*$ defined by 
\eqref{eqn:Bayes-risk}.

\subsection{Definitions and assumptions}
\label{subsec:Definitions_Assumptions}

We derive the dual representation of the learning algorithm in Figure
\ref{fig:Kernel-based_learning_algorithm}. 
For a convex function $\ell:\Rbb\rightarrow\Rbb$, let $\ell^*$ be the conjugate function
of $\ell$. 
For $o\in\{p,n\}$, suppose that the uncertainty sets are described as the form of
\begin{align}
 \Ucal_{o}[c]
 &=
 \bigg\{\sum_{i\in{M_o}}\alpha_i k(\cdot,x_i^{(1)})\in\Hcal
 \,:\,
 \sum_{i\in{M_o}}\alpha_i=1,\,\alpha_i\geq0\,(i\in{M_{o}}),\,
 \frac{1}{m}\sum_{i\in{M_{o}}}\ell^*(m\alpha_i)\leq c \bigg\}. 
 \label{eqn:kernel-based-uncertainty-set}
\end{align}
In the same way as the derivation in Section
\ref{sec:Uncertainty_Sets_Corresponding_Loss}, 
we find that the problem \eqref{eqn:opt_prob_kernel_estimation_RCM}  is 
the dual representation of 
\begin{align}
 \begin{array}{l}
  \displaystyle
  \min_{f,b,\rho}-2\rho+\frac{1}{m_1}\sum_{i=1}^{m_1}\ell(\rho-y_i^{(1)}(f(x_i^{(1)})+b))\\
  \displaystyle
   \ \st\ f\in\Hcal,\,b\in\Rbb,\,\rho\in\Rbb,\,\|f\|_{\Hcal}^2\leq\lambda^2. 
 \end{array}
 \label{eqn:loss-min_RKHS}
\end{align}
Later on, we show a rigorous proof of the duality between 
\eqref{eqn:loss-min_RKHS} and \eqref{eqn:opt_prob_kernel_estimation_RCM} with 
the uncertainty set \eqref{eqn:kernel-based-uncertainty-set}. 
In order to investigate statistical properties of the learning algorithm using uncertainty
sets, we consider the primal problem 
\eqref{eqn:loss-min_RKHS} and \eqref{eqn:opt_prob_bias_estimation} 
instead of the dual problem 
\eqref{eqn:opt_prob_kernel_estimation_RCM} and \eqref{eqn:opt_prob_bias_estimation}. 

We define some notations. 
For a measurable function $f:\Xcal\rightarrow\Rbb$ and a real number $\rho\in\Rbb$, 
we define the expected loss $\Rcal(f,\rho)$ and the regularized expected loss 
$\Rcal_\lambda(f,\rho)$ by
\begin{align*}
  \Rcal(f,\rho)&=-2\rho+\Ebb[\ell(\rho-yf(x))], \\
 \Rcal_\lambda(f,\rho)&=-2\rho+\Ebb[\ell(\rho-yf(x))]+\theta(\|f\|_{\Hcal}^2\leq\lambda^2),
\end{align*}
where $\lambda$ is a positive number and $\theta(A)$ equals $0$ when $A$ is true and 
$\infty$ otherwise. Let $\Rcal^*$ be the infimum of $\Rcal(f,\rho)$, 
\begin{align*}
 \Rcal^*=\inf\{\Rcal(f,\rho)\,:\,f\in{}L_0,\,\rho\in\Rbb\}. 
\end{align*}
For the set of training samples, $T=\{(x_1,y_1),\ldots,(x_m,y_m)\}$, 
the empirical loss $\widehat{\Rcal}_T(f,\rho)$ and the regularized empirical loss
$\widehat{\Rcal}_{T,\lambda}(f,\rho)$ are defined by
\begin{align*}
 \widehat{\Rcal}_T(f,\rho)
 &=
 -2\rho+\frac{1}{m}\sum_{i=1}^{m}\ell(\rho-y_if(x_i)),\\
 \widehat{\Rcal}_{T,\lambda}(f,\rho)
 &=
 -2\rho+\frac{1}{m}\sum_{i=1}^{m}\ell(\rho-y_if(x_i))
 +\theta(\|f\|_{\Hcal}^2\leq\lambda^2). 
\end{align*}
The subscript $T$ is dropped if it is clear from the context. 

For the observed training samples $T_1=\{(x_i^{(1)},y_i^{(1)}):i=1,\ldots,,m_1\}$, 
clearly the problem \eqref{eqn:loss-min_RKHS} is identical to the minimization of
$\widehat{\Rcal}_{T_1,\lambda}(f,\rho)$. 
We define $\widehat{f},\widehat{b}$ and 
$\widehat{\rho}$ as an optimal solution of  
\begin{align}
 \min_{f,b,\rho}\widehat{\Rcal}_{T_1,\lambda_{m_1}}(f+b,\rho),\quad
 f\in\Hcal,\, b\in\Rbb,\, \rho\in\Rbb, 
 \label{eqn:empirical_regularized_loss_minimization}
\end{align}
where the regularization parameter $\lambda_{m_1}$ may depend on the sample size. 
For the index sets $M_p$ and $M_n$ in Figure \ref{fig:Kernel-based_learning_algorithm},  
we define $m_p=|M_p|$ and $m_n=|M_n|$. 

We introduce the following assumptions. 
\begin{assumption}[universal kernel]
 \label{assump:universal_kernel}
 The input space $\Xcal$ is a compact metric space. 
 The kernel function $k:\Xcal^2\rightarrow\Rbb$ is continuous, and satisfies
 \begin{align*}
  \sup_{x\in\Xcal}\sqrt{k(x,x)}\leq{K}<\infty, 
 \end{align*}
 where $K$ is a positive constant. In addition, $k$ is universal, i.e., 
 the RKHS associated with $k$ is dense in the set of all continuous functions on $\Xcal$
 with respect to the supremum norm 
 \citep[Definition 4.52]{steinwart08:_suppor_vector_machin}. 

\end{assumption}

\begin{assumption}[non-deterministic assumption]\quad
 \label{assump:non-deterministic-assumption}
 For the probability distribution of training samples, 
 there exists a positive constant $\varepsilon>0$ such that
 \begin{align*}
  P(\{x\in\Xcal:\varepsilon\leq{}P(+1|x)\leq{}1-\varepsilon\})>0
 \end{align*} 
 holds, where $P(y|x)$ is the conditional probability of the label $y$ 
 for given input $x$. 
\end{assumption}

\begin{assumption}[basic assumptions on the loss function]
 \label{assumption:expectedloss_consistency}
 The loss function $\ell:\Rbb\rightarrow\Rbb$ satisfies the following conditions. 
 \begin{enumerate}
  \item $\ell$ is a non-decreasing, convex function, and satisfies
	the non-negativity condition, i.e., $\ell(z)\geq0$ for all $z\in\Rbb$. 
  \item Let $\partial\ell(z)$ be the subdifferential of the loss function $\ell$ at 
	$z\in\Rbb$ \citep[Chap.~23]{book:Rockafellar:1970}. Then, the equality 
	$\lim_{z\rightarrow\infty}\partial\ell(z)=\infty$ holds, 
	i.e., 
	for any $M>0$, there exists $z_0$ such that for all $z\geq{}z_0$ and all
	$g\in\partial\ell(z)$, the inequality $g\geq M$ holds. 
 \end{enumerate}
\end{assumption}
Note that the second condition in Assumption \ref{assumption:expectedloss_consistency}
assures that $\ell$ is not constant function and that
$\lim_{z\rightarrow\infty}\ell(z)=\infty$ holds. 

\begin{assumption}[modified classification-caliblated loss]\quad  
\label{assumption:loss_BayesRisk_consistency}
 \begin{enumerate}
  \item $\ell(z)$ is first order differentiable for $z\geq-\ell(0)/2$, and 
	$\ell'(z)>0$ holds for $z\geq-\ell(0)/2$, where $\ell'$ is the derivative of
	$\ell$. 
  \item Let $\psi(\theta,\rho)$ be the function defined as
	\begin{align*}
	 \psi(\theta,\rho)=\ell(\rho)-\inf_{z\in\Rbb}
	 \left\{
	 \frac{1+\theta}{2}\ell(\rho-z)
	 +\frac{1-\theta}{2}\ell(\rho+z)
	 \right\},\quad 0\leq\theta\leq{1},\ \rho\in\Rbb. 
	\end{align*}
	There exist a function $\widetilde{\psi}(\theta)$ and a positive real $\varepsilon>0$ 
	such
	that the following conditions are satisfied: 
	\begin{enumerate}
	 \item $\widetilde{\psi}(0)=0$ and $\widetilde{\psi}(\theta)>0$ for 
	       $0<\theta\leq\varepsilon$. 
	 \item $\widetilde{\psi}(\theta)$ is a continuous and strictly increasing function 
	       on the interval $[0,\varepsilon]$. 
	 \item The inequality 
	       $\displaystyle\widetilde{\psi}(\theta)\leq
	       \inf_{\rho\geq-\ell(0)/2}\psi(\theta,\rho)$
	       holds for $0\leq\theta\leq\varepsilon$. 
	\end{enumerate}
 \end{enumerate}
\end{assumption}
Later on, we shall give some sufficient conditions for existence of the function
$\widetilde{\psi}$ in Assumption \ref{assumption:loss_BayesRisk_consistency}. 

We prove that there is no duality gap between \eqref{eqn:opt_prob_kernel_estimation_RCM}
and \eqref{eqn:loss-min_RKHS}. 
The proof of the following lemma is given in Appendix
\ref{appendix:Proof_Lemma_existence_opt_sol}. 
\begin{lemma}
 \label{lemma:existence_opt_sol}
 Suppose that both $M_p$ and $M_n$ in Figure \ref{fig:Kernel-based_learning_algorithm} are
 non-empty, i.e., $m_p$ and $m_n$ are positive numbers. 
 Under Assumption \ref{assump:universal_kernel} 
 and \ref{assumption:expectedloss_consistency}, 
 there exists an optimal solution for \eqref{eqn:loss-min_RKHS}. 
 Moreover, the dual problem of \eqref{eqn:loss-min_RKHS} yields 
 the problem \eqref{eqn:opt_prob_kernel_estimation_RCM} with the uncertainty set 
 \eqref{eqn:kernel-based-uncertainty-set}. 
\end{lemma}

In the following, we prove the convergence of the error rate to the Bayes risk 
$\Ecal^*$. The proof consists of two parts. In Section
\ref{subsec:Convergence+to+Optimal_Expected_Loss}, we prove that the expected loss for the
estimated decision function, $\Rcal(\widehat{f}+\widehat{b},\widehat{\rho})$, converges to
the infimum of the expected loss $\Rcal^*$, where $\widehat{f}, \widehat{b}$ and
$\widehat{\rho}$ are optimal solutions of
\eqref{eqn:empirical_regularized_loss_minimization}. 
Here, we apply the mathematical tools developed by 
\cite{steinwart05:_consis_of_suppor_vector_machin}. 
In Section \ref{subsec:Convergence_to_Bayes_Risk}, 
we prove the convergence of the error rate $\Ecal(\widehat{f}+\widetilde{b})$ to the
Bayes risk $\Ecal^*$, where $\widetilde{b}$ is an optimal solution of
\eqref{eqn:opt_prob_bias_estimation}. 
In the proof, the concept of the classification-calibrated loss 
\citep{bartlett06:_convex_class_risk_bound} plays an important role.

\subsection{Convergence to Optimal Expected Loss}
\label{subsec:Convergence+to+Optimal_Expected_Loss}
In this section, we prove that $\Rcal(\widehat{f}+\widehat{b},\widehat{\rho})$ converges
to $\Rcal^*$. 
Following lemmas show the relation between the expected loss and the regularized the expected loss. 
Proofs are shown in Appendix \ref{appendix:proof_lemmas_risk_converge}. 
\begin{lemma}
 \label{lemma:risk_boundedness}
 Under 
 Assumption \ref{assump:non-deterministic-assumption} and 
 Assumption \ref{assumption:expectedloss_consistency}, 
 we have $\Rcal^*>-\infty$. 
\end{lemma}

\begin{lemma}
 \label{lemma:convergence_regularized_exp_loss}
 Under Assumption \ref{assump:universal_kernel}, \ref{assump:non-deterministic-assumption} 
 and \ref{assumption:expectedloss_consistency}, 
 we have 
 \begin{align}
  \lim_{\lambda\rightarrow\infty}\inf\{ \Rcal_\lambda(f,\rho):f\in\mathcal{H},\,\rho\in\Rbb \}=\Rcal^*. 
  \label{eqn:convergence_ecpected_loss}
 \end{align}
\end{lemma}

We derive an upper bound on the norm of the optimal solution in 
\eqref{eqn:empirical_regularized_loss_minimization}. 
The proof is deferred to Appendix \ref{appendix:proof_lemmas_risk_converge}. 
\begin{lemma}
 \label{lemma:estimator_bound}
 Under Assumption \ref{assump:universal_kernel},
 \ref{assump:non-deterministic-assumption}
 and \ref{assumption:expectedloss_consistency}, 
 there are positive constants $c$ and $C$ and a natural number $M$ such that 
 the optimal solution of \eqref{eqn:empirical_regularized_loss_minimization} satisfies 
 \begin{align}
  \|\widehat{f}\|_{\Hcal}\leq{}\lambda_{m_1},\quad 
  |\widehat{b}|\leq{}C\lambda_{m_1} ,\quad 
  |\widehat{\rho}|\leq{}C\lambda_{m_1}
  \label{eqn:bounds_estimators}
 \end{align}
 with the probability greater than $1-e^{-cm_1}$ for $m_1\geq{M}$. 
\end{lemma}


Let us define the covering number for a metric space. 
\begin{definition}[covering number]
 For a metric space $\Gcal$, the covering number of $\Gcal$ is defined as 
 \begin{align*}
  \Ncal(\Gcal,\varepsilon)=\min\big\{
  n\in\Nbb\,:\,
  g_1,\ldots,g_n\in\Gcal\ \text{\rm such that}\ 
  \Gcal\subset\bigcup_{i=1}^n{}B(g_i,\varepsilon)  \big\}, 
 \end{align*}
 where $B(g,\varepsilon)$ denotes the closed ball with center $g$ and radius
 $\varepsilon$. 
\end{definition}

According to Lemma \ref{lemma:estimator_bound}, 
the optimal solution, $\widehat{f},\,\widehat{b}$ and $\widehat{\rho}$, 
is included in the set 
\begin{align*}
 \Gcal_{m_1}=\{(f,b,\rho)\in\mathcal{H}\times\Rbb^2:
 \|f\|_{\Hcal}\leq{}\lambda_{m_1}, |b|\leq{}C\lambda_{m_1}, |\rho|\leq{}C\lambda_{m_1} \}
\end{align*}
with high probability. 
Suppose that the norm $\|f\|_{\infty}+|b|+|\rho|$ is introduced on $\Gcal_{m_1}$. 
We define the function 
\begin{align*}
 L(x,y;f,b,\rho)=-2\rho+\ell(\rho-y(f(x)+b)), 
\end{align*}
and the function set
\begin{align*}
 \Lcal_{m_1}=\{L(x,y;f,b,\rho):(f,b,\rho)\in\Gcal_{m_1}\}. 
\end{align*}
The supremum norm is defined on $\Lcal_{m_1}$. 
The expected loss and the empirical loss, $\Rcal(f+b,\rho)$ and
$\widehat{\Rcal}_{T_1}(f+b,\rho)$, are represented as the expectation of $L(x,y;f,b,\rho)$
with respect to the population distribution and the empirical distribution, respectively. 
Since $\ell:\Rbb\rightarrow\Rbb$ is a finite-valued convex function, $\ell$ is locally
Lipschitz continuous. Then, for any sample size $m_1$, 
there exists a constant $\kappa_{m_1}$ depending on $m_1$ such that
\begin{align}
 |\ell(z)-\ell(z')|\leq\kappa_{m_1}|z-z'|
 \label{eqn:Lipschitz_constant}
\end{align}
holds for all $z$ and $z'$ satisfying $|z|,|z'|\leq (K+2C)\lambda_{m_1}$. 
Then, for any $(f,b,\rho),(f',b',\rho')\in\Gcal_{m_1}$, we have 
\begin{align*}
 |L(x,y;f,b,\rho)-L(x,y;f',b',\rho')|
 &\leq 
 2|\rho-\rho'|+\kappa_{m_1}(|\rho-\rho'|+|b-b'|+\|f-f'\|_{\infty})\\
 &\leq 
 (2+\kappa_{m_1})(|\rho-\rho'|+|b-b'|+\|f-f'\|_{\infty})
\end{align*}
The covering number of $\Lcal_{m_1}$ is evaluated by using that of $\Gcal_{m_1}$ as follows:
\begin{align}
 \Ncal(\Lcal_{m_1,}\varepsilon)\leq\Ncal\big(\Gcal_{m_1},\frac{\varepsilon}{2+\kappa_{m_1}}\big). 
 \label{eqn:covering_bound_L_G}
\end{align}
Let the metric space $\Fcal_{m_1}$ be 
\begin{align*}
\Fcal_{m_1}=\{f\in\Hcal:\|f\|_{\Hcal}\leq{}\lambda_{m_1}\}
\end{align*}
with the supremum norm, then, we also have
\begin{align}
 \Ncal\left(\Gcal_{m_1},\frac{\varepsilon}{2+\kappa_{m_1}}\right)
 \leq 
 \Ncal\left(\Fcal_{m_1}, \frac{\varepsilon}{3(2+\kappa_{m_1})}\right)
 \left(\frac{6C\lambda_{m_1}(2+\kappa_{m_1})}{\varepsilon}\right)^2. 
 \label{eqn:covering-upper-bound_G_F}
\end{align}
An upper bound of the covering number of $\Fcal_{m_1}$ is given by 
\cite{cucker02} and \cite{zhou02:_cover_number_in_learn_theor}. 

We prove the uniform convergence of $\widehat{\Rcal}(f+b,\rho)$. 
The proof is deferred to Appendix \ref{appendix:proof_lemmas_risk_converge}. 
\begin{lemma}
 \label{lemma:uniform_convergence_loss}
 Let $b_{m_1}$ be
 \begin{align*}
  b_{m_1}=4C\lambda_{m_1}+\ell((K+2C)\lambda_{m_1})
 \end{align*}
 in which $C$ is the positive constant defined in Lemma \ref{lemma:estimator_bound}. 
 Under Assumption \ref{assump:universal_kernel} and
 \ref{assumption:expectedloss_consistency}, 
 the inequality 
 \begin{align}
  &\phantom{\leq}
  P\bigg(
  \! \sup_{(f,b,\rho)\in\Gcal_{m_1}}\!\!\!
  |\widehat{\Rcal}(f+b,\rho)-\Rcal(f+b,\rho)|\geq \varepsilon
  \bigg) \nonumber\\
  &\leq 
  2\Ncal(\Lcal_{m_1},\varepsilon/3)\exp\bigg\{-\frac{2m_1\varepsilon^2}{9b_{m_1}^2}\bigg\}
  \label{eqn:uniform_bound_L}\\
  &\leq 
  2 \Ncal\left(\Fcal_{m_1}, 
  \frac{\varepsilon}{9(2+\kappa_{m_1})}\right)
 \left(\frac{18C\lambda_{m_1}(2+\kappa_{m_1})}{\varepsilon}\right)^2
  \exp\bigg\{-\frac{2m_1\varepsilon^2}{9b_{m_1}^2}\bigg\}
  \label{eqn:uniform_bound_F}
 \end{align}
 holds, where $\kappa_{m_1}$ is the Lipschitz constant defined by 
 \eqref{eqn:Lipschitz_constant}. 
\end{lemma}

We present the main theorem of this section. 
The proof is given in Appendix \ref{appendix:proof_risk_convergence}. 
\begin{theorem}
 \label{theorem:surrogate_risk-convergence}
 Suppose that $\lim_{m_1\rightarrow\infty}\lambda_{m_1}=\infty$ holds. 
 Suppose that Assumption \ref{assump:universal_kernel},
 \ref{assump:non-deterministic-assumption}
 and \ref{assumption:expectedloss_consistency} hold. 
 Moreover we assume that \eqref{eqn:uniform_bound_F} converges to zero 
 for any $\varepsilon>0$, 
 when the sample size $m_1$ tends to infinity. 
 Then, $\Rcal(\widehat{f}+\widehat{b},\widehat{\rho})$ converges to 
 $\Rcal^*$ in probability in the large sample limit of the dataset $T_1$. 
\end{theorem}

We show the order of $\lambda_{m_1}$ admitting the assumption in Theorem
\ref{theorem:surrogate_risk-convergence}. 
\begin{example}
 Suppose that $\Xcal=[0,1]^n\subset\Rbb^n$ and the Gaussian kernel is used. 
 According to \cite{zhou02:_cover_number_in_learn_theor}, we have
 \begin{align*}
  \log\Ncal\bigg(\Fcal_{m_1},\frac{\varepsilon}{9(2+\kappa_{m_1})}\bigg)
  =
  O\left(
  \bigg(\log\frac{\lambda_{m_1}}{\frac{\varepsilon}{9(2+\kappa_{m_1})}}\bigg)^{n+1}
  \right)
  =O\big(\big(\log(\lambda_{m_1}\kappa_{m_1})\big)^{n+1}\big). 
 \end{align*}
 For any $\varepsilon>0$, \eqref{eqn:uniform_bound_F} is bounded above by 
 \begin{align*}
  \exp\left\{
  O\bigg(
  -\frac{m_1}{b_{m_1}^2}
  +(\log(\lambda_{m_1}\kappa_{m_1}))^{n+1}
  \bigg)
  \right\}. 
 \end{align*}
 For the truncated quadratic loss, we have 
 \begin{align*}
  \kappa_{m_1}&\leq 2((K+2C)\lambda_{m_1}+1)=O(\lambda_{m_1}),\\
  b_{m_1}&\leq 4C\lambda_{m_1}+((K+2C)\lambda_{m_1}+1)^2=O(\lambda_{m_1}^2). 
 \end{align*}
 Let us define $\lambda_{m_1}=m_1^\alpha$ with $0<\alpha<1/4$. 
 Then, for any $\varepsilon>0$, \eqref{eqn:uniform_bound_F} converges to zero when $m_1$
 tends to infinity. In the same way, for the exponential loss we obtain 
 \begin{align*}
  \kappa_{m_1}=O(e^{(K+2C)\lambda_{m_1}}),\quad{}  b_{m_1}=O(e^{(K+2C)\lambda_{m_1}}). 
 \end{align*}
 Hence, $\lambda_{m_1}=(\log{m_1})^\alpha$ with $0<\alpha<1$ assures the convergence of 
 \eqref{eqn:uniform_bound_F}. 
\end{example}

\subsection{Convergence to Bayes Risk}
\label{subsec:Convergence_to_Bayes_Risk}
We study the error rate of the estimated classifier. 
Let us define $\widehat{f},\widehat{b}$ and $\widehat{\rho}$ be a minimizer of 
$\Rcal_{T_1,\lambda_{m_1}}(f+b,\rho)$. 
In the proposed learning algorithm in Figure \ref{fig:Kernel-based_learning_algorithm}, 
the estimated bias term $\widehat{b}$ is replaced with $\widetilde{b}$ which is
an optimal solution of $\min_{b\in\Rbb}\widehat{\Ecal}_{T_2}(\widehat{f}+b)$. 
We prove that
the expected 0-1 loss $\Ecal(\widehat{f}+\widetilde{b})$
converges to the Bayes risk $\Ecal^*$, 
when the sample sizes of $T_1$ and $T_2$ tend to infinity.
The proof is shown in Appendix \ref{appendix:proof_bayeserror_convergence}. 
\begin{theorem}
 \label{theorem:bayes_risk_consistency}
 Suppose that $\Rcal(\widehat{f}+\widehat{b},\widehat{\rho})$ converges to
 $\Rcal^*$ in probability, 
 when the sample size of $T_1$, i.e., $m_1$, tends to infinity. 
 For the RKHS $\Hcal$ and the loss function $\ell$, 
 we assume Assumption 
 \ref{assump:universal_kernel}, \ref{assumption:expectedloss_consistency} and 
 \ref{assumption:loss_BayesRisk_consistency}. 
 Then, $\Ecal(\widehat{f}+\widetilde{b})$ converges to $\Ecal^*$ in probability, when
 the sample sizes of $T_1$ and $T_2$ tend to infinity. 
\end{theorem}
As a result, we find that the prediction error rate of $\widehat{f}+\widetilde{b}$
converges to the Bayes risk under Assumption 
\ref{assump:universal_kernel},  
\ref{assump:non-deterministic-assumption}, 
\ref{assumption:expectedloss_consistency} and
\ref{assumption:loss_BayesRisk_consistency}. 

We present some sufficient conditions for existence of the function $\widetilde{\psi}$ in
Assumption \ref{assumption:loss_BayesRisk_consistency}. 
The proof of the following lemma is shown in Appendix
\ref{appendix:proof_sufficient_cond_psi}. 
\begin{lemma}
 \label{lemma:existence_t_psi_reciplocal_l}
 Suppose that
 the first condition in Assumption \ref{assumption:expectedloss_consistency} and
 the first condition in Assumption \ref{assumption:loss_BayesRisk_consistency} hold. 
 In addition, suppose that $\ell$ is first-order continuously differentiable on $\Rbb$. 
 Let $d$ be $d=\sup\{z\in\Rbb:\ell'(z)=0\}$, where $\ell'$ is the derivative of
 $\ell$. When $\ell'(z)>0$ holds for all $z\in\Rbb$, we define $d=-\infty$. 
 We assume the following conditions: 
 \begin{enumerate}
  \item $d<-\ell(0)/2$. 
  \item  $\ell(z)$ is second-order continuously differentiable on the open interval
	 $(d,\infty)$. 
  \item $\ell''(z)>0$ holds on $(d,\infty)$. 
  \item $1/\ell'(z)$ is convex on $(d,\infty)$. 
 \end{enumerate}
 Then, for any $\theta\in[0,1]$, the function $\psi(\theta,\rho)$ is non-decreasing as the
 function of $\rho$ for $\rho\geq-\ell(0)/2$. 
\end{lemma}
When the condition in Lemma \ref{lemma:existence_t_psi_reciplocal_l} is satisfied, 
we can choose $\psi(\theta,-\ell(0)/2)$ as $\widetilde{\psi}(\theta)$ for 
$0\leq\theta\leq{1}$, since $\psi(\theta,-\ell(0)/2)$ is classification-calibrated under
the first condition in Assumption \ref{assumption:loss_BayesRisk_consistency}. 

We give another sufficient condition for existence of the function $\widetilde{\psi}$ in
Assumption \ref{assumption:loss_BayesRisk_consistency}. 
The proof of the following lemma is shown in Appendix
\ref{appendix:proof_sufficient_cond_psi}. 
\begin{lemma}
 \label{lemma:existence_psi_less_differentiable}
 Suppose that
 the first condition in Assumption \ref{assumption:expectedloss_consistency} and
 the first condition in Assumption \ref{assumption:loss_BayesRisk_consistency} hold. 
 Let $d$ be $d=\sup\{z\in\Rbb:\partial\ell(z)=\{0\}\}$. 
 When $0\not\in\partial\ell'(z)$ holds for all $z\in\Rbb$, we define $d=-\infty$. 
 Suppose that the inequality $-\ell(0)/2>d$ holds. 
 For $\rho\geq-\ell(0)/2$ and $z\geq0$, we define $\xi(z,\rho)$ by 
 \begin{align*}
  \xi(z,\rho)
  =
  \begin{cases}
   \displaystyle
   \frac{\ell(\rho+z)+\ell(\rho-z)-2\ell(\rho)}{z\ell'(\rho)}, & z>0,\\
   \displaystyle 0,& z=0. 
  \end{cases}
 \end{align*}
 Suppose that there exists a function $\bar{\xi}(z)$ for $z\geq0$ 
 such that the following conditions hold: 
 \begin{enumerate}
  \item $\bar{\xi}(z)$ is continuous and strictly increasing on $z\geq0$, and satisfies
	$\bar{\xi}(0)=0$ and $\lim_{z\rightarrow\infty}\bar{\xi}(z)>1$. 
  \item $\sup_{\rho\geq-\ell(0)/2}\xi(z,\rho)\leq \bar{\xi}(z)$ holds. 
 \end{enumerate}
 Then, there exists a function $\widetilde{\psi}$ 
 defined in the second condition of Assumption
 \ref{assumption:loss_BayesRisk_consistency}. 
\end{lemma}
Note that Lemma \ref{lemma:existence_psi_less_differentiable} does not require 
the second order differentiability of the loss function. 
We show some examples in which the existence of $\widetilde{\psi}$ is confirmed 
from the above lemmas. 

\begin{example}
 For the truncated quadratic loss $\ell(z)=(\max\{z+1,0\})^2$, 
 the first condition in Assumption \ref{assumption:expectedloss_consistency} and
 the first condition in Assumption \ref{assumption:loss_BayesRisk_consistency} hold. 
 The inequality $-\ell(0)/2=-1/2>\sup\{z:\ell'(z)=0\}=-1$ in the sufficient condition 
 of Lemma \ref{lemma:existence_t_psi_reciplocal_l} holds. 
 For $z>-1$, it is easy to see that $\ell(z)$ is second-order differentiable 
 and that $\ell''(z)>0$ holds. 
 In addition, for $z>-1$, $1/\ell'(z)$ is equal to $1/(2z+2)$ which is convex on 
 $(-1,\infty)$. 
 Therefore, the function $\widetilde{\psi}(\theta)=\psi(\theta,-1/2)$ satisfies
 the second condition in Assumption \ref{assumption:loss_BayesRisk_consistency}. 
\end{example}

\begin{example}
 For the exponential loss $\ell(z)=e^z$, we have $1/\ell'(z)=e^{-z}$. 
 Hence, 
 due to Lemma \ref{lemma:existence_t_psi_reciplocal_l}, 
$\psi(\theta,\rho)$ is non-decreasing in $\rho$. Indeed, we have
 $\psi(\theta,\rho)=(1-\sqrt{1-\theta^2})e^\rho$. 
\end{example}

\begin{example}
 In Example \ref{exam:uncertaity-set_with_error}, we presented the uncertainty set with
 estimation errors. 
 The uncertainty sets are defined based on the revised function $\bar{\ell}(z)$ in 
 \eqref{eqn:loss-for-uncertainty-set-with-error}. 
 Here, 
 we use a similar function defined by
 \begin{align}
  \bar{\ell}^*(\alpha)=
  \begin{cases}
  (|\alpha{w}-1|+h)^2-(1+h)^2, & \alpha\geq0,\\
   \infty,                     & \alpha<0,
  \end{cases}
 \end{align}
 for the construction of uncertainty sets. 
 Here, $w$ and $h$ are positive constants, and we suppose $w>1/2$. 
 The corresponding loss function is given as $\bar{\ell}(z)$. 
 Then we have $\bar{\ell}(z)=u(z/w)$ defined in
 \eqref{eqn:uo_for_uncertaintyset_est_err}. 
 For $w>1/2$, we can confirm that $\sup\{z:\bar{\ell}'(z)=0\}<-\bar{\ell}(0)/2$ holds. 
 Since $u(z)$ is not strictly convex, Lemma \ref{lemma:existence_t_psi_reciplocal_l} does
 not work. Hence, we apply Lemma \ref{lemma:existence_psi_less_differentiable}. 
 A simple calculation yields that 
 $\bar{\ell}'(-\bar{\ell}(0)/2)\geq{}(4w-1)/(4w^2)>0$ for any $h\geq{0}$. 
 Note that $\bar{\ell}(z)$ is differentiable on $\Rbb$. Thus, the monotonicity of $\bar{\ell}'$ for
 the convex function leads to 
 \begin{align*}
  \xi(z,\rho)
  =\frac{1}{\bar{\ell}'(\rho)}\left(\frac{\bar{\ell}(\rho+z)-\bar{\ell}(\rho)}{z}-\frac{\bar{\ell}(\rho)-\bar{\ell}(\rho-z)}{z}\right)
  \leq
  \frac{\bar{\ell}'(\rho+z)-\bar{\ell}'(\rho-z)}{\bar{\ell}'(\rho)}. 
 \end{align*}
 Figure \ref{fig:derivative-revised_loss} depicts the derivative of $\bar{\ell}$ with
 $h=1$ and $w=1$. 
 Since the derivative $\bar{\ell}'(z)$ is Lipschitz continuous and the 
 Lipschitz constant is equal to $1/(2w)$, we have
 $\bar{\ell}'(\rho+z)-\bar{\ell}'(\rho-z)\leq{z/w}$. Therefore, the inequality 
\begin{align*}
 \sup_{\rho\geq-\bar{\ell}(0)/2}\xi(z,\rho)
 \leq{}
 \sup_{\rho\geq-\bar{\ell}(0)/2} \frac{z/w}{\bar{\ell}'(\rho)}
 =
 \frac{z/w}{\bar{\ell}'(-\bar{\ell}(0)/2)}
 \leq
 \frac{4w}{4w-1}z
 \leq
 2z
\end{align*}
 holds. We see that $\bar{\xi}(z)=2z$ satisfies the sufficient condition of Lemma
\ref{lemma:existence_psi_less_differentiable}. 
 The inequality 
 \begin{align*}
  \bar{\ell}'(-\bar{\ell}(0)/2)\frac{\theta}{2}\bar{\xi}^{-1}(\frac{\theta}{2})\geq{}
  \frac{4w-1}{32w^2}\theta^2
 \end{align*}
 ensures that $\widetilde{\psi}(\theta)=\frac{4w-1}{32w^2}\theta^2$ is a valid choice. 
 Therefore, the loss function corresponding to the revised uncertainty set in 
 Example \ref{exam:uncertaity-set_with_error} satisfies the sufficient conditions for the
 Bayes risk consistency. 
 \begin{figure}[tb]
 \begin{center}
  \includegraphics[scale=0.5]{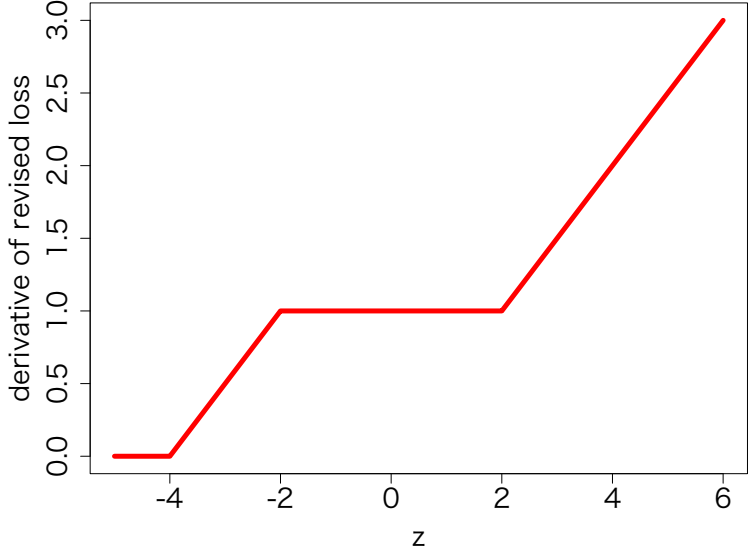}
 \end{center}
  \caption{The derivative of the loss function corresponding to the revised uncertainty
  set with the estimation error. }
  \label{fig:derivative-revised_loss}
 \end{figure}
\end{example}

\section{Experiments}
\label{sec:Numerical_Studies}
We compare the statistical properties of the proposed learning algorithm to the other
learning methods. As proved in Section \ref{sec:Statistical_Properties}, 
the kernel-based learning algorithm in Figure  \ref{fig:Kernel-based_learning_algorithm}
has the statistical consistency under some assumptions, while MPM and MM-MPM do not have
the statistical consistency in general. The main purpose of the numerical study is to
compare our method to MPM and its variants. 

We compare the kernel-based learning algorithms using the Gaussian kernel. 
So far, many works have been devoted to compare the linear models and the kernel-based
models. The conclusion is that the linear model outperforms the kernel-based model
when the decision boundary is well approximated by the linear model. 
Otherwise, the linear model has the approximation bias, and the kernel-based estimators
with a nice regularization outperform the linear models in general. 
Hence, we focus on the kernel-based estimators. 
In our experiments, the following methods were examined to the synthetic data and the
standard benchmark datasets: $C$-SVM, MPM, unbiased MPM, and 
the kernel variant of the proposed method presented in Figure
\ref{fig:simple_learning_algorithm}. 
For simplicity, the function part $f\in\mathcal{H}$ and the bias term $b\in\Rbb$ are
estimated based on all training samples, 
though in the learning algorithm in Figure \ref{fig:Kernel-based_learning_algorithm},  
the dataset is decomposed into two subsets in order to ensure the statistical
consistency. 
In the unbiased MPM, the bias term $b$ in the model is estimated by minimizing the
training error rate 
after estimating the function part, $\widehat{f}\in\mathcal{H}$. 
Clearly, the unbiased estimator will outperform the original MPM, 
when the probability of the class label is heavily unbalanced. 
In the proposed method, we apply the uncertainty set defined from the loss function $u(z)$
defined in \eqref{eqn:uo_for_uncertaintyset_est_err}. 
This is the revised uncertainty set of the ellipsoidal uncertainty set with the estimation
error. The parameter in the function $u(z)$ of \eqref{eqn:uo_for_uncertaintyset_est_err}
is set to $h=0$ or $h=1$. 
The kernel parameter and the regularization parameter are estimated by $5$-fold cross
validation. We use the test error for the evaluation of the prediction accuracy.

\subsection{Synthetic data}
\label{subsec:Synthetic_data}
Suppose that the input points $\x$ conditioned on the positive label are generated by the
two dimensional normal distribution with the mean $\mubold_p=(0,0)^T$ and the covariance 
matrix $\Sigma_p=I$, where $I$ is the identity matrix. 
In the same way, 
the conditional distribution of input points with the negative label is 
defined as the normal distribution with $\mubold_n=(1,1)^T$ and the covariance matrix
$\Sigma_n=R^T\mathrm{diag}(0.5^2,1.5^2)R$, where $R$ is the $\pi/3$ radian
counterclockwise rotation matrix. 
The label probability is defined by $P(Y=+1)=0.2$ or $0.5$. 
The size of training samples is $m=400$. 

Table \ref{tbl:normal_test_error} shows the test error of the estimators: 
$C$-SVM, MPM, unbiased MPM, learning with the loss function 
\eqref{eqn:uo_for_uncertaintyset_est_err} with $h=0$ or $h=1$. 
We notice that, under the unbalanced samples, i.e., the case of $P(Y=+1)=0.2$, 
the MPM has the estimation bias. 
On the setup of the balanced data, MPM is slightly  better than the other methods. All the
learning algorithm except MPM are comparable to each other. The difference of the
parameter $h$ in the loss function 
\eqref{eqn:uo_for_uncertaintyset_est_err} is not significant in this experiment. 

\begin{table*}[tb]
 \caption{Test error $(\%)$ of each learning method is presented with the standard deviation. 
 We compared $C$-SVM, MPM, unbiased MPM, learning method with the loss function 
 \eqref{eqn:uo_for_uncertaintyset_est_err} with $h=0$ and $h=1$. } 
\label{tbl:normal_test_error}
\centering
 \begin{tabular}{c|ccccc}
  \multicolumn{1}{l|}{$P(Y\!\!=\!+1)$}
   & $C$-SVM & MPM & unbiased MPM & $h=0$ & $h=1$ \\ 
  \hline
  0.2 & $15.8\pm1.1$ & $26.0\pm2.2$ & $16.5\pm1.2$ & $15.9\pm1.1$ & $16.0\pm1.2$ \\
  0.5 & $25.2\pm1.1$ & $25.1\pm1.0$ & $25.5\pm1.3$ & $25.5\pm1.4$ & $25.4\pm1.1$ \\
 \end{tabular}
\end{table*}

\subsection{Benchmark data}
\label{subsec:Benchmark_data}
In this section, we use thirteen artificial and real world datasets from the UCI, DELVE,
and STATLOG benchmark repositories: 
{\tt banana}, {\tt breast-cancer}, {\tt diabetes}, {\tt german}, 
{\tt heart}, {\tt image}, {\tt ringnorm}, {\tt flare-solar}, 
{\tt splice}, {\tt thyroid}, {\tt titanic}, {\tt twonorm}, 
{\tt waveform}. 
All datasets are provided as IDA benchmark repository. 
See \cite{mach:raetsch+onoda+mueller:2001} and \cite{ratsch00:_robus} for details of 
datasets. The properties of each dataset are shown in Table \ref{table:datasets}, where
``dim'', ``$P(Y=+1)$'',``\#train'', ``\#test'' and ``rep.'' denote the input dimension,
the ratio of the positive labels in training samples, the size of training set, 
the size of test set, and the number of replication of learning to evaluate the average
performance, respectively. 

In the experiment, especially we compare unbiased MPM and our method using the loss
function \eqref{eqn:uo_for_uncertaintyset_est_err} with $h=0$.  
The uncertainty set of unbiased MPM is ellipsoid defined by the estimated covariance matrix. 
The corresponding loss function of the form of \eqref{eqn:general-loss} does 
not exist, since the convex-hull of the input points is not taken into account. 
In our method using the loss function \eqref{eqn:uo_for_uncertaintyset_est_err} with
$h=0$, the uncertainty set is the intersection of the same ellipsoid as unbiased MPM
and the convex-hull of the input vectors. 
That is, the revision of the ellipsoidal uncertainty set in unbiased MPM leads to the
uncertainty set of our algorithm. 
We use the $t$-test to detect the difference of test errors of these two learning
algorithms. 

Table \ref{tbl:normal_test_error} shows test errors $(\%)$ for benchmark datasets with the
standard deviation. 
We show the results of $C$-SVM, MPM, unbiased MPM, learning method with the loss function 
\eqref{eqn:uo_for_uncertaintyset_est_err} with $h=0$ and $h=1$. 
In the column of the unbiased MPM and our method with $h=0$, 
the bold face letters indicates that the test error is smaller compared to the opponent at
the significance level $1\%$. Overall, $C$-SVM performs better than the others. 
the learning method with the loss function \eqref{eqn:uo_for_uncertaintyset_est_err} with
$h=1$ is comparable to $C$-SVM except {\tt breast-cancer}, {\tt flare-solar} and {\tt titanic}. 
Note that the loss function \eqref{eqn:uo_for_uncertaintyset_est_err} with $h=1$ is similar to
the hinge loss around zero. Hence, it is clear that the results of our method with
$h=1$ is close to the results of $C$-SVM. 
The results of $t$-test indicates that, comparing to unbiased MPM, our method using the
loss function \eqref{eqn:uo_for_uncertaintyset_est_err} with $h=0$ achieves the smaller 
test errors. In both algorithms, the same estimator is used for the bias term in the
decision function. Hence, the result implies that our method is superior to unbiased MPM 
in the estimation of the function part $f\in\mathcal{H}$ in the decision function. 
In the dataset {\tt flare-solar} and {\tt titanic}, unbiased MPM is superior to our method 
with $h=0$. 
This is because there are many duplications in covariates of these datasets. 
Indeed, in 666 training samples of {\tt flare-solar}, 
there are only 76 different input points, and {\tt titanic} has only 11 different input
points out of 150 training samples. In the other datasets, the variety of the covariates
is almost equal to the size of the training samples. In our method, the uncertainty set
for such data does not capture the distribution of the input points appropriately. 
We notice that the revision of the uncertainty set will be useful to achieve high
prediction accuracy in comparison to (unbiased) MPM, 
as long as the covariate does not have many duplications.

\begin{table}[t]
 \caption{The properties of each data sets are shown, where ``dim'',
 ``$P(Y=+1)$'',``\#train'', ``\#test'' and ``rep.'' denote the input dimension, the ratio
 of the positive label in training samples, the size of training set, the size of test
 set, and the number of replication of learning, respectively. }
 \label{table:datasets}
 \centering
 \begin{tabular}{lccccc}
  dataset        & dim & $P(Y\!\!=\!+1)$ &\#train  & \#test & rep.\\ \hline
  banana         & 2   &  0.454  &400      & 4900   & 100 \\
  breast-cancer  & 9   &  0.294  &200      &   77   & 100 \\
  diabetis       & 8   &  0.350  &468      &  300   & 100 \\
  flare-solar    & 9   &  0.552  &666      &  400   & 100 \\
  german         & 20  &  0.301  &700      &  300   & 100 \\
  heart          & 13  &  0.445  &170      &  100   & 100 \\
  image          & 18  &  0.574  &1300     & 1010   &  20 \\
  ringnorm       & 20  &  0.497  &400      & 7000   & 100 \\
  splice         & 60  &  0.483  &1000     & 2175   &  20 \\
  thyroid        &  5  &  0.305  &140      &   75   &  85 \\
  titanic        &  3  &  0.322  &150      & 2051   & 100 \\
  twonorm        & 20  &  0.505  &400      & 7000   & 100 \\
  waveform       & 21  &  0.331  &400      & 4600   & 100 \\ \hline
 \end{tabular}
\end{table}

\begin{table*}[tb]
 \caption{Test errors $(\%)$ for benchmark datasets are presented with the standard
 deviation. We compared $C$-SVM, MPM, unbiased MPM, learning method with the loss function 
 \eqref{eqn:uo_for_uncertaintyset_est_err} with $h=0$ and $h=1$. 
 We conduct $t$-test to compare the unbiased MPM and the learning method 
 using the loss function \eqref{eqn:uo_for_uncertaintyset_est_err} with $h=0$. 
 The bold face letters indicates that the test error is smaller compared to the opponent 
 at the significance level $1\%$.}
\label{tbl:normal_test_error}
\centering
 \vspace*{2mm}
 \begin{tabular}{l|rrrrr}
  \multicolumn{1}{l|}{dataset} & 
  \multicolumn{1}{c}{$C$-SVM}  &
  \multicolumn{1}{c}{MPM}      &
  \multicolumn{1}{c}{unbiased MPM}  &
  \multicolumn{1}{c}{$h=0$}   &
  \multicolumn{1}{c}{$h=1$}  \\  \hline
 banana       &
  $10.7\pm0.6$ & $11.4\pm0.9$ & $11.4\pm0.9$ & ${\bf 11.1\pm0.9}$ &  $10.9\pm0.7$ \\ 
 breast-cancer &
  $26.9\pm4.8$ & $35.0\pm4.9$ & $34.0\pm4.8$ & ${\bf 28.1\pm5.0}$ & $28.1\pm4.5$ \\
 diabetis      &
$23.9\pm 2.1$& $28.8\pm 2.4$& $28.3\pm 2.5$  & ${\bf 24.3\pm 1.9}$& $24.2\pm 2.1$\\
 flare-solar   &
$33.7\pm 2.2$& $34.9\pm 1.7$& ${\bf 35.7\pm 1.9}$& $36.8\pm 3.1$& $36.8\pm 2.9$\\
 german        &
 $23.8\pm 2.3$&  $29.2\pm 2.4$&  $28.2\pm 2.7$ &  ${\bf 23.5\pm 2.3}$&  $23.6\pm 2.4$\\
 heart         &
 $16.7\pm 3.5$&  $25.6\pm 4.2$&  $25.7\pm 4.0$&  ${\bf 17.3\pm 3.7}$&  $17.2\pm 3.5$\\
 image         &
 $3.3\pm 0.7$&  $3.2\pm 0.7$&  $3.2\pm 0.7$&  $3.4\pm 0.6$&  $3.3\pm 0.5$\\
 ringnorm      &
 $1.7\pm 0.3$&  $3.2\pm 0.4$&  $2.8\pm 0.5$&  ${\bf 1.7\pm 0.3}$&  $1.6\pm 0.2$\\
 splice        &
$11.1\pm 0.7$& $12.3\pm 1.7$& $11.7\pm 0.8$& $11.3\pm 0.7$& $11.1\pm 0.8$\\
 thyroid       &
 $5.3\pm 2.1$&  $6.3\pm 3.1$&  $6.2\pm 3.7$&  $5.6\pm 2.4$&  $5.4\pm 2.2$\\
 titanic       &
 $22.4\pm 0.8$&  $24.1\pm 2.2$&  ${\bf 22.4\pm 1.2}$&  $23.5\pm 1.6$&  $23.7\pm 3.4$\\
 twonorm       & 
  $2.6\pm 0.3$&   $4.5\pm 0.7$&   $4.4\pm 0.6$&  ${\bf 2.6\pm 0.3}$&   $2.6\pm 0.4$\\
 waveform      &
 $10.2\pm 0.7$&  $13.0\pm 0.9$&  $12.7\pm 0.8$&  ${\bf 10.2\pm 0.6}$&  $10.1\pm 0.7$\\
 \end{tabular}
\end{table*}

\section{Conclusion}
\label{sec:Conclusion}
 In this paper, we studied the relation between the loss function approach and the
 uncertainty set approach in binary classification problems. We showed that these two
 approaches are connected to each other by the conjugate property based on the Legendre
 transformation. 
 Given a loss function, there exists a corresponding parametrized uncertainty set.  
 In general, however, uncertainty set does not correspond to the empirical loss function. 
 We presented a way of revising the uncertainty set such that there exists an empirical
 loss function. 
 Then,we proposed a modified maximum-margin algorithm based on the parametrized uncertainty set. 
 We proved the statistical consistency of the learning algorithm. 
 Numerical experiments showed that the revision of the uncertainty set often improves the
 prediction accuracy of the classifier. 

 In our proof of the statistical consistency, the hinge loss used in $\nu$-SVM is excluded. 
 \cite{steinwart03:_optim_param_choic_suppor_vector_machin} 
 proved the statistical consistency of $\nu$-SVM with a nice choice of 
 the regularization parameter. 
 We are currently investigating the relaxation of the assumptions of our theoretical
 result so as to  include the hinge loss function and other popular loss functions such as
 the logistic loss. 
 As for the statistical modeling, 
 the relation between the loss function approach and the uncertainty set approach can be a
 useful tool. 
 In optimization and control theory, 
 the modeling based on the uncertainty set is frequently applied to the real-world data; 
 see the modeling in robust optimization and related works \citep{ben-tal02:_robus}. 
 We believe that 
 the learning algorithm with the revision of the uncertainty set 
 can bridge a gap between 
 statistical modeling based on some intuition and 
 nice statistical properties of the estimated classifiers.

\section*{Acknowledgments}
TK was partially supported by Grant-in-Aid for Young Scientists (20700251). AT was
partially supported by Grant-in-Aid for Young Scientists (23710174). TS was partially
supported by MEXT Kakenhi 22700289 and the Aihara Project, the FIRST program from JSPS,
initiated by CSTP. 

 \appendix
\section{Proof of Lemma \ref{lemma:existence_opt_sol}}
\label{appendix:Proof_Lemma_existence_opt_sol}

 First, we prove the existence of an optimal solution. 
 According to the standard argument on the kernel estimator, we can restrict the function
 part $f$ to be the form of 
 \begin{align*}
  f(x)=\sum_{j=1}^{m_1}\alpha_jk(x,x_j^{(1)}). 
 \end{align*}
 Then, the problem 
 is reduced to the finite-dimensional problem, 
 \begin{align}
  \begin{array}{l}
   \displaystyle
   \min_{\alphabold,b,\rho}
   -2\rho+\frac{1}{m_1}\sum_{i=1}^{m_1}
   \ell(\rho-y_i^{(1)}(\sum_{j=1}^{m_1}\alpha_jk(x_i^{(1)},x_j^{(1)})+b))\\
   \displaystyle
    \st \sum_{i,j=1}^{m_1}\alpha_i\alpha_j k(x_i^{(1)},x_j^{(1)})\leq \lambda^2.
  \end{array}
  \label{eqn:finit-version-optprob}  
 \end{align}
 Let $\zeta_0(\alphabold,b,\rho)$ be the objective function of
 \eqref{eqn:finit-version-optprob}. 
 Let us define $\Scal$ be the linear subspace in $\Rbb^{m_1}$ spanned by the 
 column vectors of the gram matrix $(k(x_i^{(1)},x_j^{(1)}))_{i,j=1}^{m_1}$. 
 We can impose the constraint $\alphabold=(\alpha_1,\ldots,\alpha_{m_1})\in\Scal$, 
 since the orthogonal complement of $\Scal$ does not affect the objective and the
 constraint in \eqref{eqn:finit-version-optprob}. 
 We see that Assumption \ref{assump:universal_kernel} and the reproducing property yield 
 the inequality
 $\|y_i^{(1)}\sum_{j=1}^{m_1}\alpha_jk(\cdot,x_j^{(1)})\|_\infty\leq{K}\lambda$. 
 Due to this inequality and the assumptions on the function $\ell$, 
 the objective function $\zeta_0(\alphabold,b,\rho)$ is bounded below by 
 \begin{align}
  \zeta_1(b,\rho)=
  -2\rho+\frac{m_p}{m_1}\ell(\rho-b-K\lambda)+\frac{m_n}{m_1}\ell(\rho+b-K\lambda).
  \label{eqn:lower-bound-objective}
 \end{align}
 Hence, for any real number $c$, the inclusion relation 
 \begin{align}
  &\phantom{\subset}
  \bigg\{(\alphabold,b,\rho)\in\Rbb^{m_1+2}~:~\zeta_0(\alphabold,b,\rho)\leq{c},\ 
  \sum_{i,j=1}^{m_1}\alpha_i\alpha_j k(x_i^{(1)},x_j^{(1)})\leq \lambda^2,\,
  \alphabold\in\Scal
  \bigg\}\label{eqn:level-set-kernel-based-problem}\\
  &\subset
  \bigg\{(\alphabold,b,\rho)\in\Rbb^{m_1+2}~:~\zeta_1(b,\rho)\leq{c},\ 
  \sum_{i,j=1}^{m_1}\alpha_i\alpha_j k(x_i^{(1)},x_j^{(1)})\leq \lambda^2,\,
  \alphabold\in\Scal
  \bigg\}\nonumber
 \end{align}
 holds. 
 Note that the vector $\alphabold$ satisfying 
 $\sum_{i,j=1}^{m_1}\alpha_i\alpha_j k(x_i^{(1)},x_j^{(1)})\leq \lambda^2$ and 
 $\alphabold\in\Scal$ is restricted to a compact subset in $\Rbb^{m_1}$. 
 We shall prove that the subset \eqref{eqn:level-set-kernel-based-problem} is compact, 
 if they are not empty. 
 We see that the two sets above are closed subsets, 
 since both $\zeta_0$ and $\zeta_1$ are continuous. 
 By the variable change from $(b,\rho)$ to $(u_1,u_2)=(\rho-b,\rho+b)$, 
 $\zeta_1(b,\rho)$ is transformed to the convex function 
 $\zeta_2(u_1,u_2)$ defined by 
 \begin{align*}
  \zeta_2(u_1,u_2)=
  -u_1+\frac{m_p}{m_1}\ell(u_1-K\lambda)-u_2+\frac{m_n}{m_1}\ell(u_2-K\lambda). 
 \end{align*} 
 The subgradient of $\ell(z)$ diverges to infinity, when $z$ tends to infinity. 
 In addition, $\ell(z)$ is a non-decreasing and non-negative function. 
 Then, we have 
 \begin{align*}
  \lim_{|u_1|\rightarrow\infty}-u_1+\frac{m_p}{m_1}\ell(u_1-K\lambda)=\infty. 
 \end{align*}
 The same limit holds for $-u_2+\frac{m_n}{m_1}\ell(u_2-K\lambda)$. 
 Hence, the level set of $\zeta_2(u_1,u_2)$ is closed and bounded, i.e., compact.
 As a result, the level set of $\zeta_1(b,\rho)$ is also compact. 
 Therefore, the subset \eqref{eqn:level-set-kernel-based-problem} is also compact 
 in $\Rbb^{m_1+2}$. This implies that \eqref{eqn:finit-version-optprob} has an optimal
 solution. 

 Next, we prove the duality between 
 \eqref{eqn:empirical_regularized_loss_minimization} and
 \eqref{eqn:opt_prob_kernel_estimation_RCM}. 
 Since \eqref{eqn:finit-version-optprob} has an optimal solution, 
 the problem with the slack variables $\xi_i,i=1,\ldots,m_1$, 
 \begin{align*}
  \begin{array}{l}
   \displaystyle
   \min_{\alphabold,b,\rho,\xibold}
   -2\rho+\frac{1}{m_1}\sum_{i=1}^{m_1}\ell(\xi_i)\\
   \displaystyle
    \st \sum_{i,j=1}^{m_1}\alpha_i\alpha_j k(x_i^{(1)},x_j^{(1)})\leq \lambda^2,\\
   \displaystyle   
    \phantom{\st}
    \rho-y_i^{(1)}(\sum_{j=1}^{m_1}\alpha_ik(x_i^{(1)},x_j^{(1)})+b)\leq\xi_i,\,i=1,\ldots,m_1. 
  \end{array}
 \end{align*}
 also has an optimal solution and the finite optimal value. 
 In addition, the above problem clearly satisfies the Slater condition 
 \cite[Assumption 6.2.4]{book:Bertsekas+etal:2003}.  
 Indeed, at the
 feasible solution, $\alphabold=\0,b=0,\rho=0$ and $\xi_i=1,i=1,\ldots,m_1$, 
 the constraint inequalities are all inactive for positive $\lambda$. 
 Hence, Proposition 6.4.3 in \cite{book:Bertsekas+etal:2003}
 ensures that the min-max theorem holds, i.e., there is no duality gap. 
 Then, in the same way as \eqref{eqn:general-nu-svm-dual}, we obtain 
 \eqref{eqn:opt_prob_kernel_estimation_RCM}
 with the uncertainty set \eqref{eqn:kernel-based-uncertainty-set}
 as the dual problem of \eqref{eqn:empirical_regularized_loss_minimization}. 

\section{Proofs of 
 Lemmas in Section 
 \ref{subsec:Convergence+to+Optimal_Expected_Loss}}
\label{appendix:proof_lemmas_risk_converge}

We show proofs of lemmas in Section \ref{subsec:Convergence+to+Optimal_Expected_Loss}. 

\subsection{Proof of Lemma \ref{lemma:risk_boundedness}}

 Let $S\subset\Xcal$ be the subset 
 $S=\{x\in\Xcal:\varepsilon\leq{}P(+1|x)\leq{}1-\varepsilon\}$, then we have $P(S)>0$. 
 Due to the non-negativity of the loss function $\ell$, we have
\begin{align*}
 \Rcal(f,\rho)
 &\geq 
 -2\rho
 + \int_S{}\bigg\{
 P(+1|x)\ell(\rho-f(x))+P(-1|x)\ell(\rho+f(x))
 \bigg\}P(dx) \\
 &=
 \int_S{}\bigg\{
 -\frac{2}{P(S)}\rho+P(+1|x)\ell(\rho-f(x))+P(-1|x)\ell(\rho+f(x))\bigg\}P(dx). 
\end{align*}
 For given $\eta$ satisfying $\varepsilon\leq\eta\leq{}1-\varepsilon$, 
 we define the function $\xi(f,\rho)$ by 
\begin{align*}
 \xi(f,\rho)=-\frac{2}{P(S)}\rho+\eta\ell(\rho-f)+(1-\eta)\ell(\rho+f),\quad
 f,\rho\in\Rbb. 
\end{align*}
 We derive a lower bound $\inf\{\xi(f,\rho):f,\rho\in\Rbb\}$. 
 Since $\ell(z)$ is a finite-valued convex function on $\Rbb$, 
 the subdifferential $\partial{\xi}(f,\rho)\subset\Rbb^2$ is given as 
 \begin{align*}
  \partial\xi(f,\rho)
  =\left\{
  (0,-\frac{2}{P(s)})^T+u\eta(-1,1)^T+v(1-\eta)(1,1)^T:
  u\in\partial\ell(\rho-f),\,v\in\partial\ell(\rho+f)\right\}. 
 \end{align*}
 Formulas of the subdifferential are presented in Theorem 23.8 and Theorem 23.9 of
 \cite{book:Rockafellar:1970}. 
 We prove that 
 there exist $f^*$ and $\rho^*$ such that $(0,0)^T\in\partial\xi(f^*,\rho^*)$ holds. 
 Since the second condition in Assumption \ref{assumption:expectedloss_consistency} holds 
 for the convex function $\ell$, 
 the union $\cup_{z\in\Rbb}\partial\ell(z)$ includes all the positive real numbers. 
 Hence, there exist $z_1$ and $z_2$ satisfying 
 $\frac{1}{\eta{}P(S)}\in\partial\ell(z_1)$ and
 $\frac{1}{(1-\eta){}P(S)}\in\partial\ell(z_2)$. 
 Then, for $f^*=(z_2-z_1)/2,\,\rho^*=(z_1+z_2)/2$, 
 the null vector is an  element of $\partial\xi(f^*,\rho^*)$. 
 Since $\xi(f,\rho)$ is convex in $(f,\rho)$, 
 the minimum value of $\xi(f,\rho)$ is attained at $(f^*,\rho^*)$. 
 Define $z_{\mathrm{up}}$ as a real number satisfying 
 \begin{align*}
  g>\frac{1}{\varepsilon{}P(S)},\quad \forall{g}\in\partial\ell(z_{\mathrm{up}}). 
 \end{align*}
 Since $\varepsilon\leq\eta\leq{}1-\varepsilon$ is assumed, both $z_1$ and $z_2$ 
 are less than $z_{\mathrm{up}}$ due to the monotonicity of the subdifferential. 
 Then, the inequality 
 \begin{align*}
  \xi(f,\rho)
  \geq
  -\frac{z_1+z_2}{P(S)}+\eta\ell(z_1)+(1-\eta)\ell(z_2)
  \geq
  -\frac{2z_{\mathrm{up}}}{P(S)}
 \end{align*}
 holds for all $f,\rho\in\Rbb$ and all $\eta$ such that
 $\varepsilon\leq\eta\leq{1-\varepsilon}$. 
 Hence, for any measurable function $f\in{L_0}$ and $\rho\in\Real$, we have 
 \begin{align*}
  \Rcal(f,\rho)\geq \int_S \frac{-2z_{\mathrm{up}}}{P(S)}P(dx)\geq{}-2z_{\mathrm{up}}. 
 \end{align*}
 As a result, we have $\Rcal^*\geq-2z_{\mathrm{up}}>-\infty$. 

\subsection{Proof of Lemma \ref{lemma:convergence_regularized_exp_loss}}

 Corollary 5.29 of \cite{steinwart08:_suppor_vector_machin} ensures that 
 the equality 
 \begin{align*}
  \inf\{\Ebb[\ell(\rho-yf(x))]:f\in\Hcal\}=\inf\{\Ebb[\ell(\rho-yf(x))]:f\in{}L_0\}
 \end{align*}
 holds for any $\rho\in\Rbb$. Thus, we have
 $\inf\{\Rcal(f,\rho):f\in\Hcal\}=\inf\{\Rcal(f,\rho):f\in{}L_0\}$
 for any $\rho\in\Rbb$. Then, the equality 
 \begin{align*}
  \inf\{\Rcal(f,\rho):f\in\Hcal,\,\rho\in\Rbb\}=\Rcal^*
 \end{align*} 
 holds. 
 Under Assumption \ref{assump:non-deterministic-assumption} and 
 Assumption \ref{assumption:expectedloss_consistency}, we have $\Rcal^*>-\infty$ due to
 Lemma \ref{lemma:risk_boundedness}. 
 Then, for any $\varepsilon>0$, there exist $\lambda_\varepsilon>0, f_\varepsilon\in\Hcal$
 and $\rho_\varepsilon\in\Rbb$ such that 
 $\|f_\varepsilon\|_{\Hcal}\leq\lambda_\varepsilon$ and 
 $\Rcal(f_\varepsilon,\rho_\varepsilon)\leq\Rcal^*+\varepsilon$ hold. 
 For all $\lambda\geq \lambda_\varepsilon$ we have
 \begin{align*}
  \inf\{\Rcal_{\lambda}(f,\rho):f\in\Hcal,\rho\in\Rbb\}
  \leq 
  \Rcal_{\lambda}(f_\varepsilon,\rho_\varepsilon)
  =
  \Rcal(f_\varepsilon,\rho_\varepsilon)
  \leq  
  \Rcal^*+\varepsilon. 
 \end{align*}
 On the other hand, it is clear that the inequality 
 $\Rcal^*\leq\inf\{\Rcal_{\lambda}(f,\rho):f\in\Hcal,\rho\in\Rbb\}$ 
 holds. Hence, Eq.\eqref{eqn:convergence_ecpected_loss} holds. 

\subsection{Proof of Lemma \ref{lemma:estimator_bound}}


 Under Assumption  \ref{assump:non-deterministic-assumption}, 
 the label probabilities, $P(y=+1)$ and $P(y=-1)$, are positive. 
 We assume that the inequalities
 \begin{align}
  \frac{1}{2}P(Y=+1)< \frac{m_p}{m_1},\quad  \frac{1}{2}P(Y=-1)< \frac{m_n}{m_1}
  \label{eqn:label-prob_ineq}
 \end{align}
 hold. Applying Chernoff bound, we see that 
 there exists a positive constant $c>0$ depending only on the marginal probability 
 of the label such that \eqref{eqn:label-prob_ineq} holds with the probability higher than
 $1-e^{-cm_1}$. 

 Lemma \ref{lemma:existence_opt_sol} ensures that 
 the problem \eqref{eqn:empirical_regularized_loss_minimization} has 
 optimal solutions $\widehat{f},\widehat{b},\widehat{\rho}$. 
 The first inequality in \eqref{eqn:bounds_estimators}, i.e.,
 $\|\widehat{f}\|_{\Hcal}\leq\lambda_{m_1}$, is clearly satisfied. 
 Then, we have $\|\widehat{f}\|_\infty\leq{}K\lambda_{m_1}$ from the reproducing property
 of the RKHSs. 
 The definition of the estimator and the non-negativity of $\ell$ 
 yield that 
 \begin{align*}
  -2\widehat{\rho}
  \leq 
  -2\widehat{\rho}+\frac{1}{m_1}\sum_{i=1}^{m_1}
  \ell(\widehat{\rho}-y_i^{(1)}(\widehat{f}(x_i^{(1)})+\widehat{b}))
  \leq 
  \widehat{\Rcal}_{T_1,\lambda_{m_1}}(0,0)
  =\ell(0). 
 \end{align*}
 Then, we have 
 \begin{align}
  \widehat{\rho}\geq -\frac{\ell(0)}{2}. 
  \label{eqn:rho_lower-bound}
 \end{align}
 Next, we consider the optimality condition of $\widehat{\Rcal}_{T_1,\lambda_{m_1}}$. 
 According to the calculus of subdifferential introduced in Section 23 of
 \cite{book:Rockafellar:1970}, 
 the derivative of the objective function with respect to $\rho$ leads to an optimality
 condition, 
 \begin{align*}
  0\,\in\,-2+\frac{1}{m_1}\sum_{i=1}^{m_1}
  \partial\ell(\widehat{\rho}-y_i^{(1)}(\widehat{f}(x_i^{(1)})+\widehat{b})). 
 \end{align*}
 The monotonicity and non-negativity of the subdifferential and the bound of
 $\|f\|_\infty$ lead to  
 \begin{align*}
  2
  &\geq
  \frac{1}{m_1}\sum_{i=1}^{m_1}
  \partial\ell(\widehat{\rho}-y_i^{(1)}\widehat{b}-K\lambda_{m_1})\\
  &=
   \frac{1}{m_1}\sum_{i=1}^{m_p}\partial\ell(\widehat{\rho}-\widehat{b}-K\lambda_{m_1})
  +\frac{1}{m_1}\sum_{j=1}^{m_n}\partial\ell(\widehat{\rho}+\widehat{b}-K\lambda_{m_1})\\
  &\geq
   \frac{1}{m_1}\sum_{i=1}^{m_p}\partial\ell(\widehat{\rho}-\widehat{b}-K\lambda_{m_1}). 
 \end{align*}
 The above expression means that there exist numbers in the subdifferential such that 
 the inequality holds, where  
 $\sum_{i=1}^{m_p}\partial\ell$ denotes the $m_p$-fold sum of the set $\partial\ell$. 
 Let $z_p$ be a real number satisfying $\frac{2m_1}{m_p}<\partial\ell(z_p)$, i.e., 
 all elements in $\partial\ell(z_p)$ are greater than $\frac{2m_1}{m_p}$. 
 Then, $\widehat{\rho}-\widehat{b}-K\lambda_{m_1}$ should be less than $z_p$. 
 In the same way, for $z_n$ satisfying $\frac{2m_1}{m_n}< \partial\ell(z_n)$, we have
 $\widehat{\rho}+\widehat{b}-K\lambda_{m_1}<z_n$. 
 The existence of $z_p$ and $z_n$ is guaranteed by Assumption 
 \ref{assumption:expectedloss_consistency}. Hence, the inequalities 
 \begin{align*}
  \widehat{\rho}&\leq K\lambda_{m_1}+\max\{z_p,z_n\},\\
  |\widehat{b}|&\leq \frac{\ell(0)}{2}+K\lambda_{m_1}+\max\{z_p,z_n\}
 \end{align*}
 hold, in which $\widehat{\rho}\geq -\ell(0)/2$ is used in the second inequality. 
 Define $\bar{z}$ as a real number such that
 \begin{align*}
  \forall{g}\in\partial\ell(\bar{z}),\quad
  \max\left\{\frac{4}{P(Y=+1)},\,\frac{4}{P(Y=-1)}\right\}< g. 
 \end{align*}
 Inequalities in \eqref{eqn:label-prob_ineq} lead to 
\begin{align*}
 \max\left\{\frac{2m_1}{m_p},\,\frac{2m_1}{m_n}\right\}
 < 
 \max\left\{\frac{4}{P(Y=+1)},\,\frac{4}{P(Y=-1)}\right\}. 
\end{align*}
 Hence, we can choose $\bar{z}$ satisfying $\max\{z_p,z_n\}<\bar{z}$. 
 Suppose that $\ell(0)/2\leq{}K\lambda_{m_1}+\bar{z}$ holds for $m_1\geq{M}$. 
 Then, the inequalities 
\begin{align*}
 |\widehat{\rho}|\leq 2K\lambda_{m_1}+2\bar{z},\quad 
 |\widehat{b}|\leq 2K\lambda_{m_1}+2\bar{z}, 
 \end{align*} 
 hold with the probability higher than $1-e^{-cm_1}$ for $m_1\geq{M}$. 
 By choosing an appropriate positive constant $C>0$, we obtain
 \eqref{eqn:bounds_estimators}. 

\subsection{Proof of Lemma \ref{lemma:uniform_convergence_loss}}


 Since $\|f\|_{\infty}\leq K\lambda_{m_1}$ holds for $f\in\Hcal$ such that
 $\|f\|_{\Hcal}\leq\lambda_{m_1}$, we have the following inequality 
 \begin{align*}
  &\phantom{\leq}
  \sup_{\substack{(x,y)\in\Xcal\times\Ybin\\ (f,b,\rho)\in\Gcal_{m_1}}}L(x,y;f,b,\rho)
  -\inf_{\substack{(x,y)\in\Xcal\times\Ybin\\ (f,b,\rho)\in\Gcal_{m_1}}}L(x,y;f,b,\rho)\\
  &\leq 
  2C\lambda_{m_1}+
  \sup_{\substack{(x,y)\in\Xcal\times\Ybin\\ (f,b,\rho)\in\Gcal_{m_1}}}\ell(\rho-y(f(x)+b))
  -(-2C\lambda_{m_1})\\
  &\leq   
  4C\lambda_{m_1}+\ell(C\lambda_{m_1}+K\lambda_{m_1}+C\lambda_{m_1})\\
  &=b_{m_1}. 
 \end{align*}
 In the same way as the proof of Lemma 3.4 in
 \cite{steinwart05:_consis_of_suppor_vector_machin}, 
 Hoeffding's inequality leads to the upper bound \eqref{eqn:uniform_bound_L}. 
 Eq.\,\eqref{eqn:uniform_bound_F} is the direct conclusion of
 \eqref{eqn:covering_bound_L_G} and 
 \eqref{eqn:covering-upper-bound_G_F}. 

\section{Proof of Theorem \ref{theorem:surrogate_risk-convergence}}
\label{appendix:proof_risk_convergence}

 Lemma \ref{lemma:convergence_regularized_exp_loss} assures that, 
 for any $\gamma>0$, there exists sufficiently large $M_1$ such that
 \begin{align*}
  |\inf\{\Rcal_{\lambda_{m_1}}(f+b,\rho):f\in\Hcal,\,b,\rho\in\Rbb\}-\Rcal^*|\leq\gamma
 \end{align*}
 holds for all $m_1\geq M_1$. 
 Thus, there exist $f_\gamma,b_\gamma$ and $\rho_\gamma$ such that
 \begin{align*}
  |\Rcal_{\lambda_{m_1}}(f_\gamma+b_\gamma,\rho_\gamma)-\Rcal^*|\leq 2\gamma 
 \end{align*}
 and $\|f_\gamma\|_{\Hcal}\leq\lambda_{m_1}$ hold for $m_1\geq{}M_1$. 
 Due to the law of large numbers, the inequality 
 \begin{align*}
  |\widehat{\Rcal}_{T_1}(f_\gamma+b_\gamma,\rho_\gamma)-\Rcal(f_\gamma+b_\gamma,\rho_\gamma)|\leq \gamma
 \end{align*}
 holds with high probability, say $1-\delta_{m_1}$, for $m_1\geq{M_2}$. 
 The boundedness property in Lemma \ref{lemma:estimator_bound} leads to 
 \begin{align*}
  P((\widehat{f},\widehat{b},\widehat{\rho})\in\Gcal_{m_1})\geq 1-e^{-cm_1}
 \end{align*}
 for $m_1\geq{M_3}$. In addition, by the uniform bound shown in Lemma 
 \ref{lemma:uniform_convergence_loss}, the inequality 
 \begin{align*}
  \sup_{(f,b,\rho)\in\Gcal_{m_1}}\!\!\!
  |\widehat{\Rcal}_{T_1}(f+b,\rho)-\Rcal(f+b,\rho)|\leq\gamma
 \end{align*}
 holds with probability $1-\delta'_{m_1}$. Hence, 
 the probability such that the inequality 
 \begin{align*}
  |\widehat{\Rcal}_{T_1}(\widehat{f}
  +\widehat{b},\widehat{\rho})-\Rcal(\widehat{f}+\widehat{b},\widehat{\rho})|
  \leq \gamma
 \end{align*}
 holds is higher than $1-e^{-cm_1}-\delta'_{m_1}$ for $m_1\geq{M_3}$. 
 Let $M_0$ be $M_0=\max\{M_1,M_2,M_3\}$. 
 Then, for any $\gamma>0$, 
 the following inequalities hold with probability higher than 
 $1-e^{-cm_1}-\delta'_{m_1}-\delta_{m_1}$ for $m_1\geq{M_0}$, 
 \begin{align}
  \Rcal(\widehat{f}+\widehat{b},\widehat{\rho})
  &\leq 
  \widehat{\Rcal}_{T_1}(\widehat{f}+\widehat{b},\widehat{\rho})+\gamma  \nonumber\\
  &\leq 
  \label{eqn:appendix_second_inequality_in_proof}
  \widehat{\Rcal}_{T_1}(f_\gamma+b_\gamma,\rho_\gamma)+\gamma\\
  &\leq 
  \Rcal(f_\gamma+b_\gamma,\rho_\gamma)+2\gamma\nonumber\\
  &=
  \Rcal_{\lambda_{m_1}}(f_\gamma+b_\gamma,\rho_\gamma)+2\gamma\nonumber\\
  &\leq\Rcal^*+4\gamma. \nonumber
 \end{align}
 The second inequality \eqref{eqn:appendix_second_inequality_in_proof} above is 
 given as 
 \begin{align*}
  \widehat{\Rcal}_{T_1}(\widehat{f}+\widehat{b},\widehat{\rho})
  =
  \widehat{\Rcal}_{{T_1},\lambda_{m_1}}(\widehat{f}+\widehat{b},\widehat{\rho})
  \leq 
  \widehat{\Rcal}_{T_1,\lambda_{m_1}}(f_\gamma+b_\gamma,\rho_\gamma)
  =
  \widehat{\Rcal}_{T_1}(f_\gamma+b_\gamma,\rho_\gamma). 
 \end{align*}

\section{Proof of Theorem \ref{theorem:bayes_risk_consistency}}
\label{appendix:proof_bayeserror_convergence}
 For a fixed $\rho$ such that $\rho\geq-\ell(0)/2$, the loss function $\ell(\rho-z)$ 
 is classification-calibrated \citep{bartlett06:_convex_class_risk_bound}, since 
 $\ell'(\rho)>0$ holds. 
 Hence $\psi(\theta,\rho)$ in Assumption \ref{assumption:loss_BayesRisk_consistency}
 satisfies 
 $\psi(0,\rho)=0$, 
 $\psi(\theta,\rho)>0$ for $0<\theta\leq1$, and 
 $\psi(\theta,\rho)$ is continuous and strictly increasing in $\theta\in[0,1]$. 
 In addition, for all $f\in\Hcal$ and $b\in\Rbb$, the inequality 
 \begin{align*}
  \psi(\Ecal(f+b)-\Ecal^*,\rho)\leq 
  \Ebb[\ell(\rho-y(f(x)+b))]-
  \inf_{f\in\Hcal,b\in\Rbb}\Ebb[\ell(\rho-y(f(x)+b))]
 \end{align*}
 holds. Details are presented in Theorem 1 and Theorem 2 of
 \cite{bartlett06:_convex_class_risk_bound}. 
 Here we used the equality 
 \begin{align*}
  \inf\{\Ebb[\ell(\rho-y(f(x)+b))]:f\in\Hcal,b\in\Rbb\}
  =\inf\{\Ebb[\ell(\rho-y(f(x)+b))]:f\in{}L_0,b\in\Rbb\}, 
 \end{align*}
 which is shown in Corollary 5.29 of \cite{steinwart08:_suppor_vector_machin}. 
 Hence, we have
 \begin{align*}
  \psi(\Ecal(\widehat{f}+\widehat{b})-\Ecal^*,\widehat{\rho})
  &\leq 
  \Ebb[\ell(\widehat{\rho}-y(\widehat{f}(x)+\widehat{b}))]-
  \inf_{f\in\Hcal,b\in\Rbb}\Ebb[\ell(\widehat{\rho}-y(f(x)+b))]\\
  &=
  \Rcal(\widehat{f}+\widehat{b},\widehat{\rho})
  -
  \inf_{f\in\Hcal,b\in\Rbb}\Rcal(f+b,\widehat{\rho}), 
 \end{align*}
 since 
 $\widehat{\rho}\geq-\ell(0)/2$ holds due to \eqref{eqn:rho_lower-bound}. 
 We assumed that $\Rcal(\widehat{f}+\widehat{b},\widehat{\rho})$ converges to 
 $\Rcal^*$ in probability. Then, for any $\varepsilon>0$, the inequality
 \begin{align*}
 \Rcal^*
 \leq 
 \inf_{f\in\Hcal,b\in\Rbb}\Rcal(f+b,\widehat{\rho})
 \leq 
 \Rcal(\widehat{f}+\widehat{b},\widehat{\rho})
 \leq 
 \Rcal^*+\varepsilon
 \end{align*}
 holds with high probability for sufficiently large $m_1$. 
 Thus, $\psi(\Ecal(\widehat{f}+\widehat{b})-\Ecal^*,\widehat{\rho})$ converges to
 zero in probability. The inequality 
\begin{align*}
 0\leq \widetilde{\psi}(\Ecal(\widehat{f}+\widehat{b})-\Ecal^*)
 \leq 
 \psi(\Ecal(\widehat{f}+\widehat{b})-\Ecal^*,\widehat{\rho})
\end{align*}
 and the assumption on the function $\widetilde{\psi}$ 
 ensure that 
 $\Ecal(\widehat{f}+\widehat{b})$ converges to $\Ecal^*$ in probability, 
 when $m_1$ tends to infinity. 
 As a result, for any $\gamma>0$, 
 \begin{align}
  |\Ecal(\widehat{f}+\widehat{b})-\Ecal^*|\leq \gamma
  \label{eqn:convergence_E_f1}
 \end{align}
 holds with probability higher than $1-\delta_{m_1,\gamma}$ with respect to the probability
 distribution of $T_1$, where $\delta_{m_1,\gamma}$ satisfies
 $\lim_{m_1\rightarrow\infty}\delta_{m_1,\gamma}=0$ for any $\gamma>0$. 
 
 Next, we study the relation between $\widehat{f}+\widehat{b}$ and
 $\widehat{f}+\widetilde{b}$. The sample size of $T_2$ is $m_2$. 
 For any fixed $f\in\Hcal$, we define the set of 0-1 valued functions, 
 $\Scal_f=\{\ind{f(x)+b\geq0}:b\in\Rbb\}$. 
 The VC-dimension of $\Scal_f$ equals to one\footnote{See \cite{Vapnik98} for the
 definition of the VC dimension. }. 
 Indeed, for two distinct points $x,x'\in\Xcal$ such that
 $f(x)\geq{}f(x')$, the event such that $\ind{f(x)+b\geq0}=0$ and $\ind{f(x')+b\geq0}=1$ 
 is impossible. 
 Hence, for any $\varepsilon>0$ and any $f\in\Hcal$, the inequality 
 \begin{align}
  \sup_{b\in\Rbb}|\widehat{\Ecal}_{T_2}(f+b)-\Ecal(f+b)|\leq \gamma
  \label{eqn:uniform_convegence-error_rate_on_Sf}
 \end{align}
 holds with probability higher than $1-\delta_{m_2,\gamma}''$ with respect to
 the joint probability of training sample $T_2$. Note that $\delta_{m_2,\gamma}''$ depends
 only on $m_2$, $\gamma$ and the VC-dimension of $\Scal_f$. Thus, 
 $\delta_{m_2}''$ is independent of the choice of $f\in\Hcal$. 
 Remember that $\widehat{f}+\widehat{b}$ depends only on the data set $T_1$. 
 Due to the law of large numbers, the inequality 
 \begin{align*}
  |\widehat{\Ecal}_{T_2}(\widehat{f}+\widehat{b})-\Ecal(\widehat{f}+\widehat{b})|
  \leq\gamma
 \end{align*}
 holds with probability higher than $1-\delta_{m_2,\gamma}'$ with respect to the probability 
 distribution of $T_2$ conditioned on $T_1$. 
 Since the 0-1 loss is bounded, 
 it is possible to choose $\delta_{m_2,\gamma}'$ independent of $\widehat{f}$. 
 From the uniform convergence property \eqref{eqn:uniform_convegence-error_rate_on_Sf}, 
 the following inequality also holds
 \begin{align*}
  |\widehat{\Ecal}_{T_2}(\widehat{f}+\widetilde{b})-\Ecal(\widehat{f}+\widetilde{b})|
  \leq \gamma
 \end{align*}
 with probability higher than $1-\delta_{m_2,\gamma}''$ with respect to the probability 
 distribution of $T_2$ conditioned on the observation of $T_1$. 
 In addition, we have
 \begin{align*}
  \widehat{\Ecal}_{T_2}(\widehat{f}+\widetilde{b})
  \leq
  \widehat{\Ecal}_{T_2}(\widehat{f}+\widehat{b}). 
 \end{align*}
 Given the training samples $T_1$ satisfying \eqref{eqn:convergence_E_f1}, 
 the inequalities
 \begin{align*}
  \Ecal(\widehat{f}+\widetilde{b})
  \leq 
  \widehat{\Ecal}_{T_2}(\widehat{f}+\widetilde{b})+\gamma
  \leq   
  \widehat{\Ecal}_{T_2}(\widehat{f}+\widehat{b})+\gamma
  \leq     
  \Ecal(\widehat{f}+\widehat{b})+2\gamma
  \leq    
   \Ecal^*+3\gamma
  \end{align*}
 hold with probability higher than $1-\delta_{m_2,\gamma}'-\delta_{m_2,\gamma}''$
 with respect to the probability 
 distribution of $T_2$ conditioned on the observation of $T_1$. 
 Hence, as for the conditional probability, we have
 \begin{align*}
  P(\{T_2:\Ecal(\widehat{f}+\widetilde{b})\leq\Ecal^*+3\gamma\}\,|\,T_1)
  \geq 1-\delta_{m_2,\gamma}'-\delta_{m_2,\gamma}''. 
 \end{align*}
 Remember that $\delta_{m_2,\gamma}'$ and $\delta_{m_2,\gamma}''$ do not depend on $T_1$. 
 Hence, as for the joint probability of $T_1$ and $T_2$, we have
 \begin{align*}
  P(\{T_1,T_2:\Ecal(\widehat{f}+\widetilde{b})\leq\Ecal^*+3\gamma\})
  \geq (1-\delta_{m_2,\gamma}'-\delta_{m_2,\gamma}'')(1-\delta_{m_1,\gamma}). 
 \end{align*} 
 The above inequality implies that $\Ecal(\widehat{f}+\widetilde{b})$ converges to
 $\Ecal^*$ in probability, when $m_1$ and $m_2$ tend to infinity. 

\section{Proofs of 
Lemma \ref{lemma:existence_t_psi_reciplocal_l} and Lemma \ref{lemma:existence_psi_less_differentiable}}
\label{appendix:proof_sufficient_cond_psi} 

\subsection{Proof of Lemma \ref{lemma:existence_t_psi_reciplocal_l}}

 For $\theta=0$ and $\theta=1$, we can directly confirm that the lemma holds. 
 In the following, we assume $0<\theta<1$ and $\rho\geq-\ell(0)/2$. 
 We consider the following optimization problem involved in $\psi(\theta,\rho)$, 
 \begin{align}
  \inf_{z\in\Rbb}\,\frac{1+\theta}{2}\ell(\rho-z)+\frac{1-\theta}{2}\ell(\rho+z). 
  \label{eqn:opt-prob_in_psi}
 \end{align}
 The objective function is a finite-valued convex function on $\Rbb$, 
 and diverges to infinity when $z$ tends to $\pm\infty$. Hence, there exists an optimal
 solution. Let $z^*\in\Rbb$ be an optimal solution of \eqref{eqn:opt-prob_in_psi}. 
 The optimality condition is given as 
 \begin{align*}
  (1+\theta)\ell'(\rho-z^*)-(1-\theta)\ell'(\rho+z^*)=0. 
 \end{align*}
 We assumed that both $1+\theta$ and $1-\theta$ are positive and that
 $\rho\geq-\ell(0)/2>d$ holds. 
 Hence, both $\ell'(\rho-z^*)$ and $\ell'(\rho+z^*)$ should not be zero. 
 Indeed, if one of them is equal to zero, the other is also zero. 
 Hence, we have $\rho-z^*\leq d$ and $\rho+z^*\leq d$. These inequalities contradict 
 $\rho>d$. 
 Then, we have $\rho-z^*>d$ and $\rho+z^*>d$, i.e., $|z^*|<\rho-d$. 
 In addition, we have
 \begin{align*}
  \frac{1+\theta}{2}=\frac{\ell'(\rho+z^*)}{\ell'(\rho+z^*)+\ell'(\rho-z^*)}. 
 \end{align*} 
 Since $\ell''(z)>0$ holds on $(d,\infty)$, 
 the second derivative of the objective in 
 \eqref{eqn:opt-prob_in_psi} satisfies the positivity condition, 
 \begin{align*}
  (1+\theta)\ell''(\rho-z)+(1-\theta)\ell''(\rho+z)>0
 \end{align*}
 for all $z$ such that $\rho-z>d$ and $\rho+z>d$. 
 Therefore, $z^*$ is uniquely determined. 
 For a fixed $\theta\in(0,1)$, 
 the optimal solution can be described as the function of $\rho$, i.e.,
 $z^*=z(\rho)$. By the implicit function theorem, $z(\rho)$ is continuously
 differentiable with respect to $\rho$. 
 Then, 
 the derivative of $\psi(\theta,\rho)$ is given as
 \begin{align*}
  \frac{\partial}{\partial\rho}\psi(\theta,\rho)
  &=
  \frac{\partial}{\partial\rho}
  \left\{
  \ell(\rho)-\frac{1+\theta}{2}\ell(\rho-z(\rho))-\frac{1-\theta}{2}\ell(\rho+z(\rho))
  \right\}\\
  &=
  \ell'(\rho)
  -\frac{1+\theta}{2}\ell'(\rho-z(\rho))\left(1-\frac{\partial{z}}{\partial\rho}\right)
  -\frac{1-\theta}{2}\ell'(\rho+z(\rho))\left(1+\frac{\partial{z}}{\partial\rho}\right)  \\
  &=  
  \ell'(\rho)
  -
  \frac{\ell'(\rho+z(\rho))}{\ell'(\rho+z(\rho))+\ell'(\rho-z(\rho))}
  \ell'(\rho-z(\rho))\left(1-\frac{\partial{z}}{\partial\rho}\right)\\
  &\phantom{=}
  -
  \frac{\ell'(\rho-z(\rho))}{\ell'(\rho+z(\rho))+\ell'(\rho-z(\rho))}
  \ell'(\rho+z(\rho))\left(1+\frac{\partial{z}}{\partial\rho}\right)  \\
  &=
  \ell'(\rho)-
  \frac{2\ell'(\rho-z(\rho))\ell'(\rho+z(\rho))}{\ell'(\rho+z(\rho))+\ell'(\rho-z(\rho))}. 
 \end{align*}
 The convexity of $1/\ell'(z)$ for $z>d$ leads to
 \begin{align*}
  0<
  \frac{1}{\ell'(\rho)}
  \leq 
  \frac{1}{2\ell'(\rho+z(\rho))}
  +\frac{1}{2\ell'(\rho-z(\rho))}
  =
  \frac{\ell'(\rho+z(\rho))+\ell'(\rho-z(\rho))}{2\ell'(\rho-z(\rho))\ell'(\rho+z(\rho))}. 
 \end{align*}
 Hence, we have
 \begin{align*}
  \frac{\partial}{\partial\rho}\psi(\theta,\rho)\geq{0}
 \end{align*}
 for $\rho\geq-\ell(0)/2>d$ and $0<\theta<1$. As a result, we see that 
 $\psi(\theta,\rho)$ is non-decreasing as the function of $\rho$. 

\subsection{Proof of Lemma \ref{lemma:existence_psi_less_differentiable}}
 We use the result of \cite{bartlett06:_convex_class_risk_bound}. 
 For a fixed $\rho$, the function $\xi(z,\rho)$ is continuous for $z\geq0$, 
 and the convexity of $\ell$ leads to the non-negativity of $\xi(z,\rho)$. 
 Moreover, the convexity and the non-negativity of $\ell(z)$ lead to 
\begin{align*}
 \xi(z,\rho)
 \geq 
 \frac{\ell(\rho+z)-\ell(\rho)}{z\ell'(\rho)}-\frac{\ell(\rho)}{z\ell'(\rho)}
 \geq 
 1-\frac{\ell(\rho)}{z\ell'(\rho)}
\end{align*}
 for $z>0$ and $\rho\geq-\ell(0)/2$, where
 $\ell(\rho)$ and $\ell'(\rho)$ are positive for $\rho>-\ell(0)/2$. 
 The above inequality and the continuity of $\xi(\cdot,\rho)$ ensure that
 there exists $z$ satisfying $\xi(z,\rho)=\theta$ for all $\theta$ such that 
 $0\leq \theta<1$. 
 We define the inverse function $\xi_\rho^{-1}$ by 
 \begin{align*}
  \xi_\rho^{-1}(\theta)=\inf\{z\geq{0}:\xi(z,\rho)=\theta\}
 \end{align*}
 for $0\leq \theta<1$. 
 For a fixed $\rho\geq-\ell(0)/2$, the loss function $\ell(\rho-z)$ is
 classification-calibrated  \citep{bartlett06:_convex_class_risk_bound}. 
 Hence, Lemma 3 in \cite{bartlett06:_convex_class_risk_bound}
 leads to the inequality 
 \begin{align*}
  \psi(\theta,\rho)\geq
  \ell'(\rho)\frac{\theta}{2}\xi_\rho^{-1}\big(\frac{\theta}{2}\big), 
 \end{align*}
 for $0\leq \theta<1$. 
 Define $\bar{\xi}^{-1}$ by 
 \begin{align*}
  \bar{\xi}^{-1}(\theta)=\inf\{z\geq{0}:\bar{\xi}(z)=\theta\}. 
 \end{align*} 
 From the definition of $\bar{\xi}(z)$, 
 $\bar{\xi}^{-1}(\theta)$ is well-defined for all $\theta\in[0,1)$. 
 Since $\xi(z,\rho)\leq \bar{\xi}(z)$ holds, 
 we have $\xi_\rho^{-1}(\theta/2)\geq\bar{\xi}^{-1}(\theta/2)$. 
 In addition, $\ell'(\rho)$ is non-decreasing as the function of $\rho$. 
 Thus, we have
 \begin{align*}
  \psi(\theta,\rho)\geq
  \ell'(-\ell(0)/2)\frac{\theta}{2}\bar{\xi}^{-1}\big(\frac{\theta}{2}\big)
 \end{align*} 
 for all $\rho\geq-\ell(0)/2$ and $0\leq\theta<1$. 
 Then, we can choose 
 \begin{align*}
  \widetilde{\psi}(\theta)
  =
  \ell'(-\ell(0)/2)\frac{\theta}{2}\bar{\xi}^{-1}\big(\frac{\theta}{2}\big). 
 \end{align*}
 It is straightforward to confirm that the conditions of Assumption
 \ref{assumption:loss_BayesRisk_consistency} are satisfied. 

\bibliographystyle{plain}

\end{document}